\documentclass[11pt]{article}    
  \usepackage[margin=0.92in]{geometry} 
  \usepackage{listings}
\usepackage[colorlinks,
citecolor=blue
]{hyperref}  
\usepackage{mathrsfs} 
\usepackage{kpfonts,comment}   
\usepackage{subcaption}

\usepackage{pifont}
\usepackage{lipsum}

\usepackage{cpdef}
\usepackage{float}
\usepackage{mathrsfs}
\usepackage{epigraph}
\usepackage{fancyvrb}
\usepackage{authblk}
\usepackage{booktabs}
\usepackage{listings}
\usepackage{xcolor}

\RecustomVerbatimCommand{\VerbatimInput}{VerbatimInput}
{fontsize=\footnotesize,
 frame=single,  
 framesep=0.5em,
 labelposition=topline,
}

 \definecolor{armygreen}{rgb}{0.29, 0.33, 0.13}
 \definecolor{cadmiumgreen}{rgb}{0.0, 0.42, 0.24}

\title{Training with (Swap) Regret Loss in a Single-Layer Self-Attention Model: A Case Study on the Probability Simplex}
\author[1]{Chanwoo Park}
\author[1]{Asuman Ozdaglar}
\affil[1]{MIT EECS}
\date{\today}
\newboolean{finalver}
\setboolean{finalver}{true} %

\ifthenelse{\boolean{finalver}}{
      \newcommand{\cp}[1]{\textcolor{purple}{\bf\small [#1 --CP]}}
    \newcommand{\kz}[1]{\textcolor{blue}{\bf\small [#1 --KZ]}}

\newcommand{\safevspace}[1]{\vspace{0mm}}
}{
  \newcommand{\cp}[1]{\textcolor{purple}{\bf\small}}
      \newcommand{\kz}[1]{}

\newcommand{\safevspace}[1]{\vspace{0mm}}
}
\newcommand{\softmax}{\operatorname{Softmax}}
\newcommand{\sm}{\text{sm}}
\newcommand{\Diag}{\text{Diag}}
\newcommand{\op}{\text{op}}

\newcommand{\proj}{\operatorname{Proj}}
\newcommand{\FixedPoint}{\operatorname{Fixed-Point}}
\newcommand{\Pperp}[1]{P_{\perp #1}}
\newcommand{\R}{\mathbb{R}}

\newcommand{\SwapReg}{\text{Swap-Regret}}
\newcommand{\Reg}{\text{Regret}}
\newcommand{\safesmallvspace}[1]{\vspace{0mm}}
\begin{document}

\maketitle
\begin{abstract}
We revisit the regret loss framework introduced in \citet{park2024llm}, which uses decision-theoretic regret as a direct loss function for training models to make better decisions, through the lens of probability-simplex policies. Our first result shows that a single-layer self-attention model trained with regret loss admits a stationary point whose forward-pass exactly matches \emph{smoothed fictitious play} with the appropriate stepsize that ensures no-regret behavior—i.e., for any given policy input, the model outputs the same update that smoothed fictitious play would produce. In parallel, we also newly introduce a swap-regret loss function, which extends the regret-loss framework beyond external regret and enables models to directly optimize for swap-deviation robustness. We further show that this swap-regret loss admits a stationary point whose forward pass implements the corresponding swap-regret update induced by classical Blum–Mansour no-swap-regret algorithm, with each head implementing an external-regret update via smoothed fictitious play.
Together, these results show that regret-trained attention can realize differentiable mechanisms whose deployment induces equilibrium behavior in games: external-regret dynamics lead to coarse correlated equilibrium, while swap-regret dynamics lead to correlated equilibrium. Thus, regret-based objectives steer minimal attention architectures toward online-learning dynamics with game-theoretic guarantees, without supervised traces of those algorithms.
\end{abstract}

\safevspace{}
\safevspace{}
\safevspace{}

\section{Introduction}
\safevspace{}
\safesmallvspace{}

As large language models (LLMs) become increasingly integrated into everyday life, their usage is evolving from single-turn, task-specific applications to multi-agent, sequential decision-making scenarios  \citep{yao2022react, hao2023reasoning, shinn2024reflexion, wang2023describe, SignificantGravitasAutoGPT,ahn2022can, wang2023voyager,li2024stride}. Rather than merely seeking one-off answers—such as solving a math problem \citep{guo2025deepseek} —users now engage with LLMs in ways that influence long-term planning and cumulative outcomes. As more individuals turn to LLMs to augment their decision-making, the setting naturally sometimes becomes multi-agent: each user constitutes an independent agent with unique goals, interacting with a shared or personalized model in pursuit of improved long-term outcomes and multi-agent sequential decision-making is \textit{regret}, which quantifies the difference between the cumulative reward achieved by an agent and that of the best possible strategy in hindsight \citep{shalev2012online}. While minimizing regret is not the only objective in these settings, it provides a rigorous measure of an agent’s adaptability and learning efficiency over time.

\safevspace{}

Recent work shows that \emph{off-the-shelf} LLMs—those used without fine-tuning or explicit inference-time control—still struggle on even simple, canonical online decision-making tasks. Empirical evaluations consistently find that such models fail to explore effectively without external scaffolding \citep{krishnamurthy2024can}, incur \emph{linear} regret in non-stochastic environments \citep{park2024llm}, exhibit unstable exploitation behavior \citep{xia2024beyond}, and underperform in non-stationary settings \citep{zhang2025comparing}. Taken together, these results indicate that fundamental decision-making abilities—regret minimization, exploration, and adaptation—remain unsolved for current LLMs. This highlights the need for principled methods that explicitly strengthen these foundations before deploying LLMs as autonomous agents at scale.

\safevspace{}

To address this gap, \citet{park2024llm} proposed a novel training loss function named \textit{regret loss}, which explicitly optimizes models to minimize regret across sequences of interactions. While their preliminary results—focused primarily on online learning tasks—demonstrated both empirical gains and theoretical guarantees, the broader applicability of regret-based training remains underexplored. Follow-up paper \citep{park2025post} found that post-training with regret-based signals can enable successful adaptation in single-agent environments, suggesting that regret supervision can play a critical role in improving sequential learning performance. Moreover, given that LLMs are fundamentally built on self-attention architectures, it is essential to understand how such regret-minimization dynamics manifest in even the simplest form of these models. 

\safevspace{}

However, a detailed characterization of regret-based training for single-layer self-attention is still incomplete even in the canonical probability-simplex setting.  \citet{park2024llm} analyzed continuous $\ell_2$-constrained policies and proved that regret loss recovers Follow-the-Regularized-Leader updates, yet it remains unclear how the same loss behaves when the policy must lie on the simplex—precisely the regime relevant to external and swap regret in online learning and games.  \Cref{tab:regret-training-comparison} summarizes this gap.

\safevspace{}

In this paper, we revisit the \emph{regret loss} framework with a singular focus on probability-simplex policies. Our first contribution shows that, for the standard loss sequences used to evaluate external regret, a stationary point of a single-layer self-attention model trained with regret loss coincides with \emph{smoothed fictitious play} with a $\tilde{\Theta}(1/\sqrt{T})$ stepsize. Our second contribution targets swap regret. We introduce a new \emph{swap-regret loss} that operationalizes swap regret as a differentiable training objective. Using this loss, we construct a multi-head linear self-attention network whose stationary point mirrors the classical Blum–Mansour no-swap-regret algorithm \citep{blum2007external}. Each head implements a smoothed fictitious play external-regret update with a $\tilde{\Theta}(1/\sqrt{T})$ stepsize, so the entire block behaves like a collection of parallel external-regret minimizers stitched together through multi-head self-attention. We prove that optimizing the swap-regret loss recovers this structure end-to-end, showing that differentiable multi-head attention can faithfully instantiate swap-regret dynamics. Together, these results highlight how external- and swap-regret objectives can be realized directly within minimal Transformer architectures.

\safevspace{}
This provides a mechanistic route from regret-based post-training objectives for LLM agents to equilibrium guarantees in multi-agent environments.  Regret-trained attention is not merely imitating an online-learning algorithm; at the stationary points we identify, it implements a differentiable mechanism whose repeated deployment induces the corresponding game-theoretic behavior.  Single-head softmax attention implements external-regret dynamics, so multi-agent deployment leads to coarse correlated equilibrium.  Multi-head fixed-point attention implements swap-regret dynamics, so deployment leads to correlated equilibrium.  The choice of regret surrogate is therefore not only an optimization detail: external-regret loss induces coarse-deviation robustness, while swap-regret loss induces action-contingent deviation robustness.

\safesmallvspace{}
\safevspace{}

\paragraph{Notation}
We write $\pmb{1}_d$ and $\pmb{0}_d$ for the $d$-dimensional all-ones and all-zeros vectors, $\pmb{O}_{d\times d}$ for the $d\times d$ zero matrix, and $I_{d\times d}$ for the $d\times d$ identity, let $\Delta(\cA)$ be the probability simplex over the action set $\cA$, use $[d] := \{1, \dots, d\}$ for the first $d$ integers, and write $\mathbbm{1}\{\cdot\}$ for indicator functions.  For asymptotics as $T\to\infty$ we use $\Theta(f(T))$ to denote quantities bounded above and below by constant multiples of $f(T)$, $O(f(T))$ for upper bounds, $o(f(T))$ for terms that vanish relative to $f(T)$, and $\tilde O(f(T))$ (or $\tilde\Theta(f(T))$) when bounds hold up to polylogarithmic factors. For a set $\Pi$ and norm $\|\cdot\|$, we denote by $\proj_{\Pi,\|\cdot\|}(x) := \argmin_{y \in \Pi} \|x - y\|$ the projection of $x$ onto $\Pi$; when the norm is Euclidean we write $\proj_\Pi(x)$ for brevity. For any vector $v$, $ \Pperp{v}$ denotes the orthogonal projector onto the subspace orthogonal to $v$. For a Markov matrix $P$, we write $\FixedPoint(P)$ for any stationary distribution satisfying $\pi = P \pi$; when the stationary distribution is unique we identify $\FixedPoint(P)$ with that $\pi$.

\safevspace{}
\safevspace{}

\section{Preliminaries}
\safevspace{}
\safevspace{}
\subsection{Linear Self-Attention}
\label{ssec:self-tf}
\safevspace{}
\safesmallvspace{}

We consider a single-layer linear attention model equipped with an output operator that maps the resulting representation into a decision policy. This architecture represents one of the simplest forms of Transformer-based models and has been used in prior work to analyze the inductive properties of attention mechanisms \citep{ahn2023transformers, mahankali2023one, park2024llm}. The output of the model is given by:

\safevspace{}
\safevspace{}
\safevspace{}

{\small
\begin{align}
    g((\ell_1, \dots, \ell_t, \pmb{1}_d); V, K, Q, v_c, k_c, q_c) = \texttt{Operator} \left( \sum_{i=1}^{t} (V \ell_i + v_c) \left( (K \ell_i + k_c)^\intercal (Q \pmb{1}_d + q_c) \right) \right), \label{eqn:single-linear-transformer}
\end{align}}
\safevspace{}
\safevspace{}

where \( V, K, Q \in \mathbb{R}^{d \times d} \) are the value, key, and query matrices, and \( v_c, k_c, q_c \in \mathbb{R}^d \) are their corresponding bias vectors. We denote the full set of model parameters as \( \phi = (V, K, Q, v_c, k_c, q_c) \).
\safevspace{}
\safevspace{}

The model takes as input a sequence of loss vectors \( (\ell_1, \dots, \ell_t) \subset \mathbb{R}^d \), along with a placeholder vector \( \pmb{1}_d \in \mathbb{R}^d \), and produces a prediction for the next policy \( \pi_{t+1} \). The inner product term computes attention scores over the sequence \( (\ell_1, \dots, \ell_t) \), and the resulting aggregated vector is transformed by a final mapping, denoted as \texttt{Operator}, to produce a valid policy. The specific choice of \texttt{Operator} depends on the policy space \( \Pi \). For example, if \( \Pi = \Delta(\cA) \), the probability simplex induced by a finite action set (as in the Experts Problem), \texttt{Operator} is the standard $\softmax$ function. In the most basic case, \texttt{Operator} can be the identity map, \( \texttt{Operator}(x) = x \), recovering the original form of linear self-attention \citep{ahn2023transformers,zhang2023trained,schlag2021linear}. 
\safevspace{}

This formulation provides a minimal yet expressive setting for studying how attention-based models can learn adaptive policies in sequential decision-making environments. In this paper, as \Cref{eqn:single-linear-transformer} has the equivalent formulation with 
\safevspace{}
\safevspace{}
\safevspace{}

\begin{align}
    g((\ell_1, \dots, \ell_t, \pmb{1}_d); V, a, {v}_c) = \texttt{Operator} \left( \sum_{i=1}^{t} ({V} \ell_i \ell_i^\intercal {a} + (V + v_c a^\intercal) \ell_i + v_c)\right), \label{eqn:single-linear-transformer-reduced}
\end{align}

\safevspace{}
\safevspace{}
where $V \in \RR^{d \times d}$ and $v_c, a \in \RR^d$. \Cref{eqn:single-linear-transformer-reduced} is a reduced form of \Cref{eqn:single-linear-transformer} but having the same representation power with \Cref{eqn:single-linear-transformer}.
\safevspace{}
\safevspace{}

\subsection{Learning Environment}

\safevspace{}
\subsubsection{Online Learning}
\label{ssec:online-learning}
\safevspace{}

We consider the \emph{online learning} framework in which an agent interacts with an environment over $T$ rounds, making sequential decisions and receiving feedback. At each round $t \in [T]$, the agent selects a decision policy $\pi_t \in \Pi$, where $\Pi$ is a bounded decision space. After committing to $\pi_t$, the environment reveals a loss function $f_t : \Pi \to [-B, B]$ for some constant $B > 0$, which may be chosen adversarially. The agent then incurs loss $f_t(\pi_t)$ and updates her policy based on this feedback.

\safevspace{}

A fundamental special case is when the action space $\mathcal{A}$ is finite and the decision set is the probability simplex over $\mathcal{A}$, i.e., $\Pi = \Delta(\mathcal{A})$. This is known as the \emph{Experts Problem} \citep{cover1966behavior, vovk1990aggregating, littlestone1994weighted, hazan2016introduction}. In this case, each loss function can be expressed as a linear form $f_t(\pi_t) = \langle \ell_t, \pi_t \rangle$, where $\ell_t \in \mathbb{R}^d$ is a loss vector (with $d := |\mathcal{A}|$). Whenever we refer to the policy space as the simplex below, we implicitly mean this finite-action setting. The agent will receive the entire loss vector $\ell_t$ in the full-information setting.

\safevspace{}
\subsubsection{Repeated Static Games}
\label{ssec:static-games}
\safevspace{}

The online learning framework extends naturally to multi-agent interactions through the lens of \emph{repeated static games}. 
Let $\mathcal{G} = \langle N, \{\cA_n\}_{n \in [N]}, \{r_n\}_{n \in [N]} \rangle$ denote a normal-form game with players $n \in [N]$, where each $\cA_n$ is the action set of player $n$, and $r_n : \cA:= \cA_1 \times \dots \times \cA_N \to [-B, B]$ is her reward (or negative loss) function.  At each round $t$, every player outputs a \emph{policy} $\pi_{n,t} \in \Delta(\cA_n)$. 
We use the term \emph{action} only for the realized play induced by the policy:
in finite (simplex) domains, an action $a_{n,t}$ is sampled as $a_{n,t} \sim \pi_{n,t}$.
From the perspective of player $n$, the realized joint play induces an individual online learning problem over $\Delta(\cA_n)$ with per-round loss
\[
\ell_{n,t} := - \EE_{a_{-n,t} \sim \pi_{-n,t}} \big[ r_n(\cdot, a_{-n,t}) \big] := - r_n(\cdot, \pi_{-n,t}).
\]

\safevspace{}

\subsection{Performance Metrics: External and Swap Regret}
\safevspace{}

\subsubsection{External Regret}
\label{ssec:regret}
\safevspace{}

We begin by defining \emph{external regret} (usually just denoted as regret), a central performance measure in online learning and repeated games. Let $\mathscr{A}$ denote an algorithm that selects a sequence of policies $(\pi_{\mathscr{A}, t})_{t \in [T]}$ over a decision space $\Pi$. Given a sequence of loss functions $(f_t)_{t \in [T]}$, the \emph{external regret} of $\mathscr{A}$ is defined as
\begin{align}
    \text{Regret}_{\mathscr{A}}((f_t)_{t \in [T]}) := \sum_{t=1}^T f_t(\pi_{\mathscr{A}, t}) - \inf_{\pi \in \Pi} \sum_{t=1}^T f_t(\pi). \label{eq:ext-regret}
\end{align}
This quantity compares the cumulative loss of the learner to that of the best fixed decision in hindsight. In the Experts Problem, this specializes to
\safevspace{}
\safevspace{}
\safevspace{}
\safevspace{}

\begin{align}
    \text{Regret}_{\mathscr{A}}((\ell_t)_{t \in [T]}) := \sum_{t=1}^T \langle \ell_t, \pi_{\mathscr{A}, t} \rangle - \inf_{\pi \in \Pi} \sum_{t=1}^T \langle \ell_t, \pi \rangle,
\end{align}

\safevspace{}
\safevspace{}
\safevspace{}

where each $\ell_t \in \mathbb{R}^d$ is a loss vector and $\Pi = \Delta(\mathcal{A})$ is the simplex over the finite action set $\mathcal{A}$. An algorithm $\mathscr{A}$ is said to be \emph{no-(external-)regret} if $
\sup_{(f_t)} \text{Regret}_{\mathscr{A}}((f_t)_{t \in [T]}) = o(T),$ 
i.e., the worst-case regret grows sublinearly in $T$. 

\safevspace{}

\subsubsection{Swap Regret}
\label{ssec:swap-regret}
\safevspace{}

While external regret compares the learner’s performance to the best fixed decision in hindsight, swap regret allows each action to be reassigned to another action in hindsight.  Formally, a swap deviation is encoded by a mapping $\sigma:[d]\to[d]$ (or equivalently by a row-stochastic matrix $P_\sigma$).  The \emph{swap regret} of algorithm $\mathscr{A}$ in the Experts setting is

\safevspace{}
\safevspace{}
\safevspace{}
\safevspace{}
\begin{align}
    \text{Swap-Regret}_{\mathscr{A}}((\ell_t)_{t \in [T]})
    := \max_{\sigma: [d] \to [d]} 
    \sum_{t=1}^T 
    \bigl\langle \pi_{\mathscr{A}, t}, \, \ell_t - P_\sigma^\intercal \ell_t \bigr\rangle
    = \max_{P \in \mathcal{M}} 
    \sum_{t=1}^T 
    \bigl\langle \pi_{\mathscr{A}, t}, \, \ell_t - P^\intercal \ell_t \bigr\rangle,
    \label{eq:swap-regret}
\end{align}
\safevspace{}
\safevspace{}
\safevspace{}

where $P_\sigma \in \{0,1\}^{d \times d}$ is defined by $(P_\sigma)_{ij} = \mathbbm{1}\{\sigma(i) = j\}$ and $\mathcal{M} := \{P \in \mathbb{R}^{d \times d} : P \pmb{1}_d = \pmb{1}_d,\, P \ge 0\}$ denotes the set of row-stochastic (Markov) matrices. 
The equality follows since the objective is linear in each row of $P$, and thus its maximum over $\mathcal{M}$ is attained at an extreme point $P_\sigma$. 
Intuitively, swap regret measures how much the learner could have improved by systematically substituting each played action $i$ with another action $\sigma(i)$ across all rounds. An algorithm is said to be \emph{no-swap-regret} if the above quantity is $o(T)$. In this paper, if $\mathscr{A}$ is clear from the context, we will drop the subscript and write $\text{Regret}$ and $\text{Swap-Regret}$ for external and swap regret, respectively.

\safevspace{}
    \subsection{Equilibrium Concepts in Repeated Games}
    \safevspace{}
    \label{ssec:cce-ce-phie}
    Repeated interactions among learning agents often lead to stable long-run behaviors 
    that correspond to equilibrium concepts.  Writing everything in terms of policies 
    allows us to reason directly about simplex-valued play.
        \safevspace{}

    \paragraph{Coarse Correlated Equilibrium (CCE).}
    A distribution $\mu \in \Delta(\cA)$ is a \emph{CCE} if, for every player $n$ and 
    every fixed (policy-level) deviation $\pi_n' \in \Delta(\cA_n)$,

\safesmallvspace{}
\safevspace{}
\safevspace{}
    \[
        \EE_{\pi \sim \mu}\!\left[ r_n(\pi) \right]
        \;\ge\;
        \EE_{\pi \sim \mu}\!\left[ r_n(\pi_n',\, \pi_{-n}) \right].
    \]
    \safevspace{}
\safevspace{}
\safevspace{}

    Thus, no player can gain by unconditionally switching to a single alternative policy
    prior to receiving any recommendation.

        \safevspace{}

    \paragraph{Correlated Equilibrium (CE).}
    A distribution $\mu \in \Delta(\cA)$ over joint actions is a \emph{correlated equilibrium} if, for every player $n$ and every action-dependent deviation represented by a row-stochastic Markov matrix $M_n \in \mathcal{M}$,

\safevspace{}
\safevspace{}
\safevspace{}
    \[
        \EE_{\pi \sim \mu}\!\left[ r_n(\pi) \right]
        \;\ge\;
        \EE_{\pi \sim \mu}\!\left[
            r_n\!\big(M_n(\pi_n),\, \pi_{-n}\big)
        \right],
    \]
    \safevspace{}
\safevspace{}
\safevspace{}

    where $M_n(\pi_n)$ denotes the (possibly mixed) action drawn from the row of $M_n$ corresponding to the recommended action~$\pi_n$. Thus, no player can gain by conditionally deviating from the recommendation using any action-dependent strategy.
    
           \safevspace{}

    If every player achieves no external regret, the empirical average play converges to a CCE, and if every player achieves no swap regret, the empirical average play converges to a CE \citep{cesa2006prediction}.

        \safevspace{}

\subsection{Concepts of Regret Loss Training}
        \safevspace{}

\label{ssec:regret-loss}

We now define a training objective that explicitly minimizes regret across sampled decision-making environments. Let the model be parameterized by $\phi$, and let $\text{model}_\phi$ denote a policy generator (e.g., a single-layer attention model as in \Cref{eqn:single-linear-transformer}). The \emph{regret loss function} is defined as
\begin{align}
    \mathcal{L}(\phi, \mathcal{D}) := \mathbb{E}_{\mathcal{D}} \left[ h\left( \Reg_{\text{model}_\phi}\left((f_t)_{t \in [T]}\right) \right) \right], \label{eqn:regret-loss}
\end{align}
where $h: \mathbb{R} \to \mathbb{R}$ is a convex surrogate function (e.g., identity, squared loss, or a smooth approximation of $\max$), and the expectation is taken over decision-making instances sampled from $\mathcal{D}$. Throughout the paper we instantiate $\Reg$ as either external regret or swap regret. We will describe the precise form of the model input shortly. The distribution $\mathcal{D}$ captures the randomness in the environment. Each round’s loss function is $f_t(\pi) = \langle \ell_t, \pi \rangle$, where $\ell_t \sim \mathcal{D}_d$ is a sampled loss vector. For each decision round $t$, the model receives the interaction history and produces the next policy $\pi_{t+1}$. Especially, the model input is the full history of loss vectors, i.e., $(\ell_1, \dots, \ell_t)$.

Previous theoretical work on regret-loss training \citep{park2024llm, park2025post} primarily analyzed continuous $\ell_2$-constrained policies. In this paper we restrict attention to the probability simplex—the setting relevant to external and swap regret—and study how regret loss shapes the induced dynamics. \Cref{tab:regret-training-comparison} summarizes the most relevant comparisons.        \safevspace{}
        \safevspace{}

\begin{table}[h]
\centering
\resizebox{\textwidth}{!}{%
\begin{tabular}{@{}llllll@{}}
\toprule
\textbf{} & \textbf{No Regret} (test) & \textbf{\texttt{Operator}}& \textbf{Post Operation} & 
\begin{tabular}[c]{@{}l@{}}
\textbf{FTRL}
\\
\textbf{Equivalence}
\end{tabular} & \textbf{FTRL stepsize}\\
\midrule
\begin{tabular}[c]{@{}l@{}}
    \cite{park2024llm}
    \\
    Theorem 5.2
    \end{tabular}
    & 
    \begin{tabular}[c]{@{}l@{}}Online Learning
    \\
    ($\Pi = B(\pmb{0}_d, R, \norm{\cdot}_2)$)
    \end{tabular} & Identity 
    & $\proj_{B(\pmb{0}_d, R, \norm{\cdot}_2)}$
    &  {\color{green}\ding{51}} & $\Theta\left(\frac{1}{\sqrt{Td}}\right)$\\
    \addlinespace 
\begin{tabular}[c]{@{}l@{}}
\cite{park2024llm}
\\
Conjecture 3
\end{tabular}
& 
\begin{tabular}[c]{@{}l@{}}Online Learning\\
($\Pi = \Delta(\cA)$)\\Static Game \end{tabular} & $\softmax$
& $\textcolor{red}{\times}$
& \begin{tabular}[c]{@{}l@{}} ?  \end{tabular} & {\color{gray} $?$}  \\
\midrule
\Cref{thm:stationary}
& \begin{tabular}[c]{@{}l@{}}Online Learning\\($\Pi = \Delta(\cA)$, no external regret) \\ Static Game\end{tabular}
& $\softmax$ 
&  $\textcolor{red}{\times}$
& \begin{tabular}[c]{@{}l@{}} {\color{green}\ding{51}}\\(Stationary Point)\end{tabular} & ${\Theta\left(\frac{1}{\sqrt{T}}\right)}$ \\
\addlinespace
\Cref{thm:new-architecture-stationary}
& \begin{tabular}[c]{@{}l@{}}Online Learning\\($\Pi = \Delta(\cA)$, no swap regret) \\ Static Game\end{tabular}
& \begin{tabular}[c]{@{}l@{}}Row-wise\\$\softmax$,\\ $\FixedPoint$\end{tabular}
&$\textcolor{red}{\times}$ &{\begin{tabular}[c]{@{}l@{}}
    {\color{green}\ding{51}}\\
    (Stationary Point)\\
    {\small Each head implements an}\\
    {\small external-regret FTRL update}
    \end{tabular}} & ${\Theta\left(\frac{1}{\sqrt{T}}\right)}$ \\
\bottomrule
\end{tabular}
}
\caption{The top rows report known results from \citet{park2024llm}; the bottom rows highlight our new stationary-point characterizations showing that regret-loss objectives on attention models recover smoothed fictitious play (external regret) and Blum–Mansour dynamics (swap regret).}
\label{tab:regret-training-comparison}

\end{table}

        \safevspace{}
        \safevspace{}
        \safevspace{}
        \safevspace{}

\section{Emergence of Smoothed Fictitious Play as a Stationary Point of Single-layer Linear Self-Attention}
        \safevspace{}

We begin by revisiting the observation of \citet{park2024llm} that single-layer linear self-attention trained with the regret-loss objective recovers FTRL updates with $L_2$ regularization when the operator is the identity. Their result, however, depends crucially on applying a projection \(\proj_{\Pi,\|\cdot\|}\) to the model output only at inference time: the training objective never sees the projection, so the match to FTRL is achieved post hoc. This leaves open whether the post operation can be integrated directly into the optimization so that the learned representation is already the desired policy without any additional post operation steps.
Throughout this section we adopt an end-to-end viewpoint: the same nonlinearity used at inference—namely the $\softmax$ mapping—is embedded directly in the loss. Consequently, the model must optimize through the $\softmax$ transformation itself, jointly learning the attention weights while satisfying the policy-domain constraints. We show that even under this stricter objective, the stationary points remain fully interpretable: they coincide with smoothed fictitious play on the simplex. Moreover, we characterize the associated theoretical stepsize of these stationary points, which in turn yields a no-regret guarantee. This is also conjectured by \citet{park2024llm} that the stationary point of the regret-loss function is the smoothed fictitious play. 

\begin{restatable}{theorem}{stationary}
    \label{thm:stationary}
The configuration of a single-layer linear attention model as defined in \Cref{eqn:single-linear-transformer-reduced}, with \texttt{Operator} = $\softmax$, and parameters $V = -k I_{d\times d}$, $a = \pmb{0}_d$, and $v_c = v \pmb{1}_d$, is a stationary point of the regret-loss function \Cref{eqn:regret-loss} if $k = \tilde{\Theta}(1/\sqrt{T})$, with $h(x) = x^2$, $f_t(\pi) = \langle \ell_t, \pi \rangle$, where $\ell_t \sim \mathcal{N}(\pmb{0}_d, I_{d\times d})$ i.i.d.
\end{restatable}
        \safevspace{}

\paragraph{Proof structure for \Cref{thm:stationary}.}
At \(V=-kI_{d\times d}\), \(a=\pmb{0}_d\), and \(v_c=v\pmb{1}_d\), the
round-\(t\) logits and policy satisfy
\[
    z_t=-k\sum_{s<t}\ell_s+(t-1)v\pmb{1}_d,\qquad
    \pi_t=\softmax\!\left(-k\sum_{s<t}\ell_s\right),
\]
where the second identity uses shift-invariance.  Thus the forward pass is
entropy-regularized FTRL, equivalently smoothed fictitious play, with learning
rate \(k\); the substantive issue is to prove that the full regret-loss
gradient vanishes at a nonzero \(k\) of the claimed scale.

Let \(R_T(k)\) be the external regret induced by the above policy and
\(L(k)=\EE[R_T(k)^2]\).  The first-order conditions are then verified block by
block: the \(v_c\)-block is annihilated pointwise by shift-invariance and
\(J_{\sm}(z)^\intercal \pmb{1}_d=\pmb{0}_d\); the \(a\)-block requires Gaussian
permutation symmetry, which makes the expected gradient a multiple of
\(\pmb{1}_d\), while the softmax Jacobian maps into the zero-sum subspace.  For
the \(V\)-block, permutation equivariance and the same row-sum constraint force
\[
    \nabla_V \mathcal L(-kI,\pmb{0},v\pmb{1})
    = \alpha(k)\left(I_{d\times d}
      -\frac{1}{d}\pmb{1}_d\pmb{1}_d^\intercal\right).
\]
Consequently the matrix condition reduces to the scalar derivative along the
ray \(V=-kI\).

It remains to locate a zero of \(L'(k)\).  The appendix derives
\[
    L'(k)=-2\Sigma_1(k)+2\Sigma_2(k),
\]
where \(\Sigma_1\) is the learner-response term and \(\Sigma_2\) is the
best-fixed-action comparator term.  Expanding
\(\pi_t(k)=u-kJ_0S_{t-1}+O(k^2\|S_{t-1}\|_\infty^2)\) and
\(J_t(k)=J_0+O(k\|S_{t-1}\|_\infty)\), with
\(u=\pmb{1}_d/d\) and \(J_0=\Diag(u)-uu^\intercal\), separates the signs of
these terms: the comparator contribution dominates for very small \(k\), giving
\(L'(k)<0\), while at the inverse-\(\sqrt T\) scale the learner-response term
dominates after Gaussian maximum, MGF, and softmax-Jacobian remainder bounds,
giving \(L'(k)>0\).  Continuity yields a root
\(k^\star=\tilde{\Theta}(1/\sqrt T)\), at which every parameter block is
first-order stationary.

The proof, deferred to \Cref{appendix:pfthm1}, pins down the precise scale of the attention weight \(k\). Earlier analyses of in-context learning only established that FTRL-like outputs could be recovered at stationary points if they have a $\softmax$ layer \citep{cheng2023transformers} or convergence of the training dynamics \citep{chen2024trainingdynamicsmultiheadsoftmax}; in contrast, \Cref{thm:stationary} identifies the first-order stationary point and shows that the resulting update is precisely smoothed fictitious play with learning rate \(\tilde{\Theta}(1/\sqrt{T})\). Empirically, we observe a unique \(k^\star(d,T)\) consistent with this scaling, further supporting the theoretical prediction (Figure~\ref{fig:softmax-square}).
\begin{figure}[!h]
    \centering
    \includegraphics[width=\textwidth]{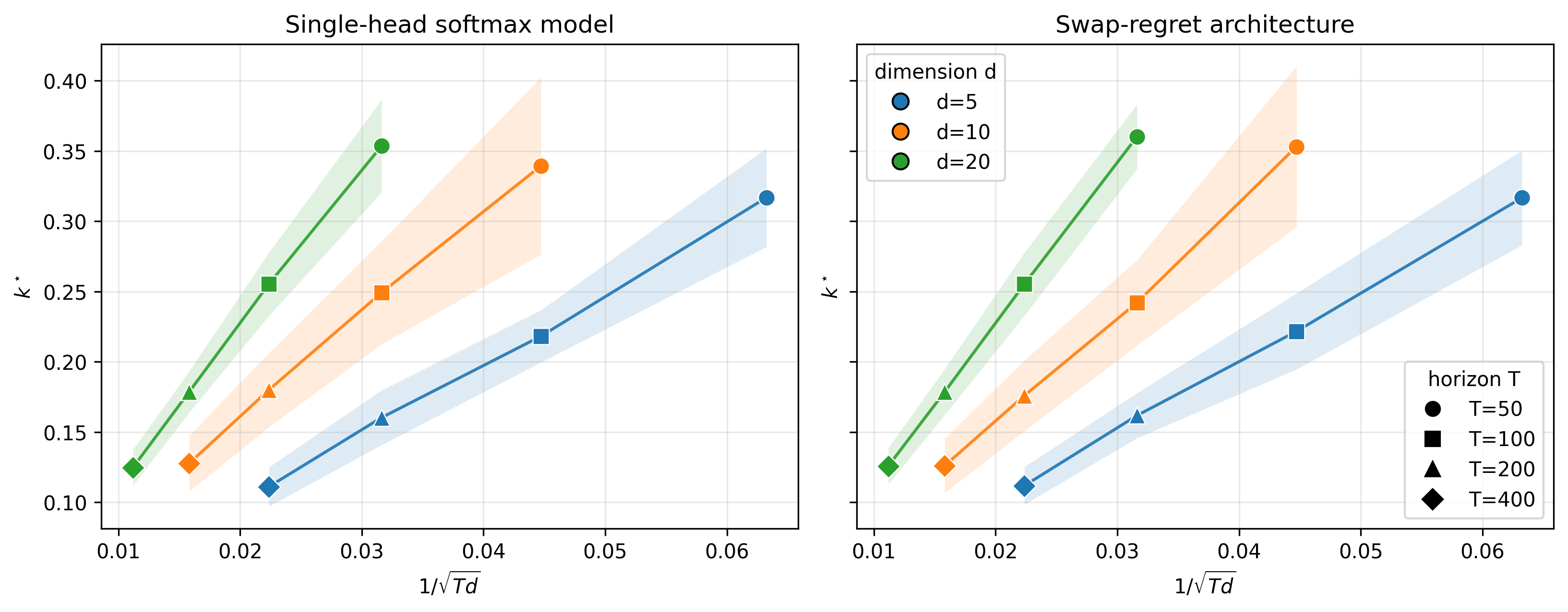}
    \caption{Estimated stationary stepsize $k^\star$ on the probability simplex as a function of $(d,T)$ for the single-head softmax model (left) and the swap-regret architecture (right). Each marker aggregates $50$ independent runs with $800$ Monte Carlo samples per run; shaded bands indicate $\pm 1$ standard deviation.}
    \label{fig:softmax-square}

        \safevspace{}
        \safevspace{}

\end{figure}

The above characterizations lift immediately to multi-agent learning dynamics.
\begin{corollary}
    Consider a repeated game where each player deploys the trained single-layer architecture from \Cref{thm:stationary} with \texttt{Operator} = $\softmax$. Then the time-averaged joint play converges to a coarse correlated equilibrium (CCE) in finite action games.
\end{corollary}
Thus, enforcing the simplex-respecting nonlinearity inside the loss not only yields interpretable single-agent updates but also preserves standard equilibrium guarantees when the model is embedded in multi-agent environments.  Mechanistically, the learned attention layer implements an external-regret dynamic, and external regret is precisely the deviation notion underlying CCE: no player can improve by replacing their realized sequence of actions with a single fixed action in hindsight.  This connects the stationary-point characterization of regret-trained attention directly to coarse-deviation robustness in games.

\subsection{Local minimality within the smoothed-fictitious-play family}
\label{ssec:additional-convergence-result}

\begin{proposition}[Restricted local minimality]
\label{prop:restricted-local-minimum}
Let
\[
  L(k):=\EE[R_T(k)^2],
  \qquad
  \pi_t(k)=\softmax\!\left(-k\sum_{s<t}\ell_s\right).
\]
Under the setting of \Cref{thm:stationary}, there exists
\(k^\star=\widetilde{\Theta}(1/\sqrt T)\) that is a local minimizer of
\(L(k)\). Consequently,
\[
  V=-k^\star I_{d\times d},\qquad
  a=\pmb{0}_d,\qquad
  v_c=v\pmb{1}_d
\]
is a local minimizer of the population regret loss when optimization is
restricted to this Softmax--smoothed-fictitious-play scalar family.
\end{proposition}

\begin{proof}
The sign estimates in the proof of \Cref{thm:stationary} provide
\(k_-<k_+\), both at the stated scale up to the displayed dimension and
polylogarithmic factors, such that \(L'(k_-)<0\) and \(L'(k_+)>0\).
The continuous function \(L\) attains a minimum on
\([k_-,k_+]\). The strict derivative signs exclude both endpoints, so at
least one minimizer \(k^\star\) lies in the interior. It is therefore a local
minimizer of \(L\), and \(L'(k^\star)=0\). The full-gradient symmetry
reduction in the proof of \Cref{thm:stationary} then also gives first-order
stationarity in the unrestricted parameterization.
\end{proof}

This proposition is deliberately restricted to the scalar FTRL family. It
does not establish local minimality over the full parameter space or global
convergence of gradient descent. We separately test whether training from
random initialization reaches the identified behavior.

\subsection{Empirical FTRL matching}
\label{ssec:empirical-ftrl-matching}

We test whether unrestricted random-initialization training reaches the
entropy-FTRL behavior identified by \Cref{thm:stationary}.  We train the
reduced one-layer model with squared regret loss at \(d=3,T=20\) for three
independent seeds, without FTRL trajectories, FTRL parameters, or a
constraint \(V=-kI\).  For each learned checkpoint, we fit a scalar
\(\eta\in[0,5]\) on \(4{,}096\) held-out Gaussian sequences by minimizing the
mean policy MSE between the model output and
\[
  \pi_t^{\mathrm{FTRL}}(\eta)
  =\softmax\!\left(-\eta\sum_{s<t}\ell_s\right).
\]

\Cref{fig:training-experiments} (left) shows that the learned policies are close
to entropy-FTRL: on average, only \(2.71\%\) of probability mass must be
moved to match the fitted FTRL policy.  This is finite-horizon evidence that
the theoretically identified family is empirically reachable; it is not a
claim of global convergence from arbitrary initialization.

\subsection{Regret training in deeper Softmax Transformers}
\label{ssec:deeper-transformer-experiment}

We next train two- and four-layer causal Softmax Transformers under the same
Gaussian regret objective.  These models use four attention heads per layer;
we do not fit their outputs to FTRL or claim an FTRL identification theorem.
We evaluate on held-out biased-Gaussian sequences
\(\ell_t=m+\epsilon_t\), where the sequence-specific mean \(m\) creates a
best action that must be learned from history.  Writing \(R_t\) for
cumulative regret through time \(t\), both depths maintain substantially
smaller \(R_t/t\) than the non-adaptive uniform policy.  At \(T=20\), the
two- and four-layer models attain respectively
\(R_{20}=4.380\pm0.046\) and \(4.322\pm0.017\), or normalized regret
\(0.219\) and \(0.216\), compared with \(R_{20}=7.021\pm0.052\) and
\(R_{20}/20=0.351\) for uniform play
(\Cref{fig:training-experiments}, right).

\begin{figure}[t]
  \centering
  \begin{minipage}[c]{0.48\linewidth}
    \centering
    \textbf{(a) One-layer FTRL matching}\\[2pt]
    {
    \scriptsize
    \begin{tabular}{@{}lc@{}}
      \toprule
      Quantity & Mean \(\pm\) SD \\
      \midrule
      Fitted \(\widehat\eta\) & \(0.4456\pm.0046\) \\
      Policy MSE & \(0.000739\pm.000176\) \\
      Mean \(L_1\) & \(0.0542\pm.0071\) \\
      TV (\% mass) & \(0.0271\pm.0036\) (\(2.71\%\)) \\
      Error/action & \(0.0181\pm.0024\) (\(1.81\) pp) \\
      \bottomrule
    \end{tabular}}
  \end{minipage}
  \hfill
  \begin{minipage}[c]{0.48\linewidth}
    \centering
    \textbf{(b) Regret over time}\\[2pt]
    \includegraphics[width=\linewidth]{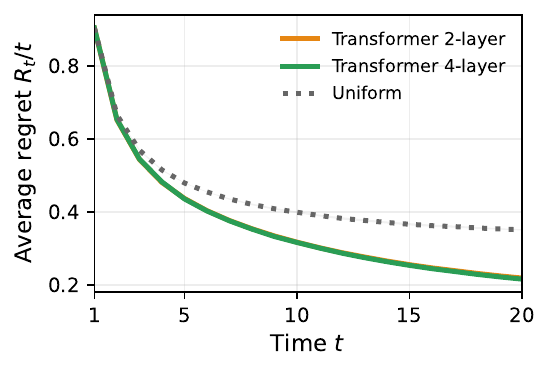}
  \end{minipage}
  \caption{Random-initialization regret-loss training at
  \(d=3,T=20\).  Left: entropy-FTRL fit to one-layer policies; ``pp'' denotes
  percentage points.  Right: normalized regret for deeper Transformers on
  biased-Gaussian sequences.  The table reports mean \(\pm\) sample standard
  deviation and curves show means with \(\pm1\) standard-error bands across
  three seeds.}
  \label{fig:training-experiments}
\end{figure}

\subsection{Coarse-correlated-equilibrium evaluation}
\label{ssec:empirical-cce}

We evaluate the learned policies in a \(3\times3\) weighted zero-sum
rock--paper--scissors game whose equilibrium marginal is
\((3,2,1)/6\).  With expected full-information loss feedback, we form the
time-averaged joint mixed play
\(\bar\mu_T=T^{-1}\sum_{t=1}^T
\pi_{1,t}\otimes\pi_{2,t}\) and measure its maximum CCE-deviation gap.
At \(T=20\), the one-, two-, and four-layer models obtain respectively
\(0.0832\pm0.0028\), \(0.0467\pm0.0095\), and
\(0.0653\pm0.0054\), compared with \(0.1111\) for uniform play.
\Cref{fig:equilibrium-experiment} (left) reports the complete curves.  This
directly measures the equilibrium quantity controlled by average external
regret.

\section{Swap-Regret Loss Minimization on the Probability Simplex}
We now recall the classical reduction of Blum and Mansour, which turns any
external-regret minimizer on the simplex into a swap-regret minimizer. Let the action set be $\cA = \{1,\dots,d\}$ and write $\Delta(\cA)$ for the corresponding probability simplex.
A \emph{swap transformation} is represented by a column-stochastic matrix
\[
\Phi \;\coloneqq\; \bigl\{P=(p_1 \mid \dots \mid p_d) \in \R^{d\times d}
  : p_i \in \Delta(\cA) \;\text{for all } i\bigr\},
\]
which acts on $\pi\in\Delta(\cA)$ by $\pi \mapsto P\pi = \sum_{i=1}^d \pi_i p_i$.
At round $t$, the adversary reveals a loss vector $\ell_t\in\R^d$ so that the learner incurs the linear loss $\langle \ell_t, \pi_t\rangle$.

\paragraph{Blum--Mansour construction.}
Fix $d$ copies of an external-regret minimizer for $\Delta(\cA)$,
denoted $\cR_1,\dots,\cR_d$.
At each round $t\in[T]$ the swap-regret algorithm performs:

\begin{enumerate}
  \item \textbf{Next strategy.}
  For each $i\in[d]$, query $\cR_i$ for  $p_{i,t}\in\Delta(\cA)$.
  Assemble the matrix
  \[
    P_t \;=\; \bigl(p_{1,t} \mid \dots \mid p_{d,t}\bigr) \in \Phi.
  \]
  Compute a fixed point $\pi_t = \FixedPoint(P_t)\in \Delta(\cA)$ and play $\pi_t$.
  \item \textbf{Loss feedback.}
  After observing the loss vector $\ell_t$,
  each external-regret copy $\cR_i$ receives the scaled linear loss $
    x_{i,t} \;=\; \pi_{t,i}\,\ell_t$,
  i.e., it runs on the loss vector $x_{i,t}$.
\end{enumerate}

Thus each $\cR_i$ runs on the same action space $\Delta(\cA)$, but with losses
rescaled by the $i$-th coordinate of the fixed point $\pi_t$.

\begin{fact}[Blum--Mansour]\label{fact:BM-swap}
Let $\Reg_i$ denote the external regret of $\cR_i$ after $T$ rounds when
fed the losses $(x_{i,t})_{t \in [T]}$, that is
\begin{equation}\label{eq:BM-Reg-i}
  \Reg_i((x_{i,t})_{t \in [T]})
  \;\coloneqq\;
  \max_{\tilde p_i \in \Delta(\cA)}
  \sum_{t=1}^T
  \bigl(\langle x_{i,t}, p_{i,t}\rangle-\langle x_{i,t}, \tilde p_i\rangle\bigr).
\end{equation}
Then the cumulative swap regret admits the exact column-wise decomposition
\begin{equation}\label{eq:BM-swap-decomposition}
  \SwapReg((\ell_t)_{t \in [T]})
  =-\min_{\hat P=(\hat p_1\mid\dots\mid\hat p_d)\in\Phi}
    \sum_{t=1}^T\sum_{i=1}^d
    \bigl\langle x_{i,t},\hat p_i-p_{i,t}\bigr\rangle,
\end{equation}
and therefore satisfies
\[
  \SwapReg((\ell_t)_{t \in [T]}) \;\le\; \sum_{i=1}^d \Reg_i((x_{i,t})_{t \in [T]}).
\]
In particular, if each $\cR_i$ has sublinear external regret, then the
Blum--Mansour construction has sublinear swap regret.
\end{fact}
The proof is deferred to \Cref{appendix:pfBM}.

Inspired by the classical Blum–Mansour construction, we introduce a new neural architecture for swap-regret minimization. Our design uses a $d$-head linear self-attention module to parameterize the transition matrix. After the multi-head attention step, each head applies a $\softmax$ to produce one strictly positive column of the transition matrix, thereby obtaining a valid column-stochastic Markov matrix. The final policy is then computed via a $\FixedPoint$ procedure, which can be efficiently implemented using power iteration. This yields a differentiable, end-to-end model that mirrors the theoretical structure of the Blum–Mansour minimizer while allowing scalable learning. Mathematically, the architecture is given by:
\begin{equation}
  \begin{aligned}
    \pi_{t+1} &= \texttt{Fixed\text{-}Point}\!\left( \left( p_{1,{t+1}} \,\mid\,\cdots\mid \, p_{d,{t+1}} \right) \right), \\
    p_{j,{t+1}} &= \softmax\!\left( \sum_{i=1}^{t} 
      \big( V^{(j)} X_i X_i^\intercal a^{(j)} 
        + (V^{(j)} + v_{c}^{(j)} a^{(j)^\intercal}) X_i 
        + v_{c}^{(j)} \big)\right),
     \qquad j \in [d].
  \end{aligned}
  \label{eqn:new-architecture}
  \end{equation}
  
where $V^{(j)} \in \RR^{d \times d^2}, a^{(j)} \in \RR^{d^2}, v_{c}^{(j)} \in \RR^d$ are the parameters for the $j$-th head of the linear self-attention module, and $X_i := \ell_i \otimes \pi_i = (x_{i,t})_{t=1}^\intercal \in \RR^{d^2}$, which consider every possible combination of $\ell_i$ and $\pi_i$. The $\FixedPoint$ operation is implemented using power iteration.

\begin{remark}[Differentiability of the fixed-point layer]
Because every column $p_{j,t+1}$ is produced by a $\softmax$, the matrix
$P_{t+1}=(p_{1,t+1}\mid\cdots\mid p_{d,t+1})$ has strictly positive entries and
is column-stochastic.  By Perron--Frobenius theory, $P_{t+1}$ is primitive and
has a unique stationary distribution in the relative interior of the simplex.
Moreover, the eigenvalue $1$ is simple, so the fixed-point map
$P\mapsto \FixedPoint(P)$ is smooth on the open set of strictly positive
column-stochastic matrices; equivalently, differentiability follows from the
implicit function theorem after restricting to the simplex tangent space
\(\{u:\pmb{1}_d^\intercal u=0\}\).  In implementation, gradients can be obtained
by backpropagating through the finite power-iteration procedure, or by implicit
differentiation of the stationary equation \(\pi=P\pi\).
\end{remark}
\begin{figure}[h]
  \centering
  \includegraphics[width=\textwidth]{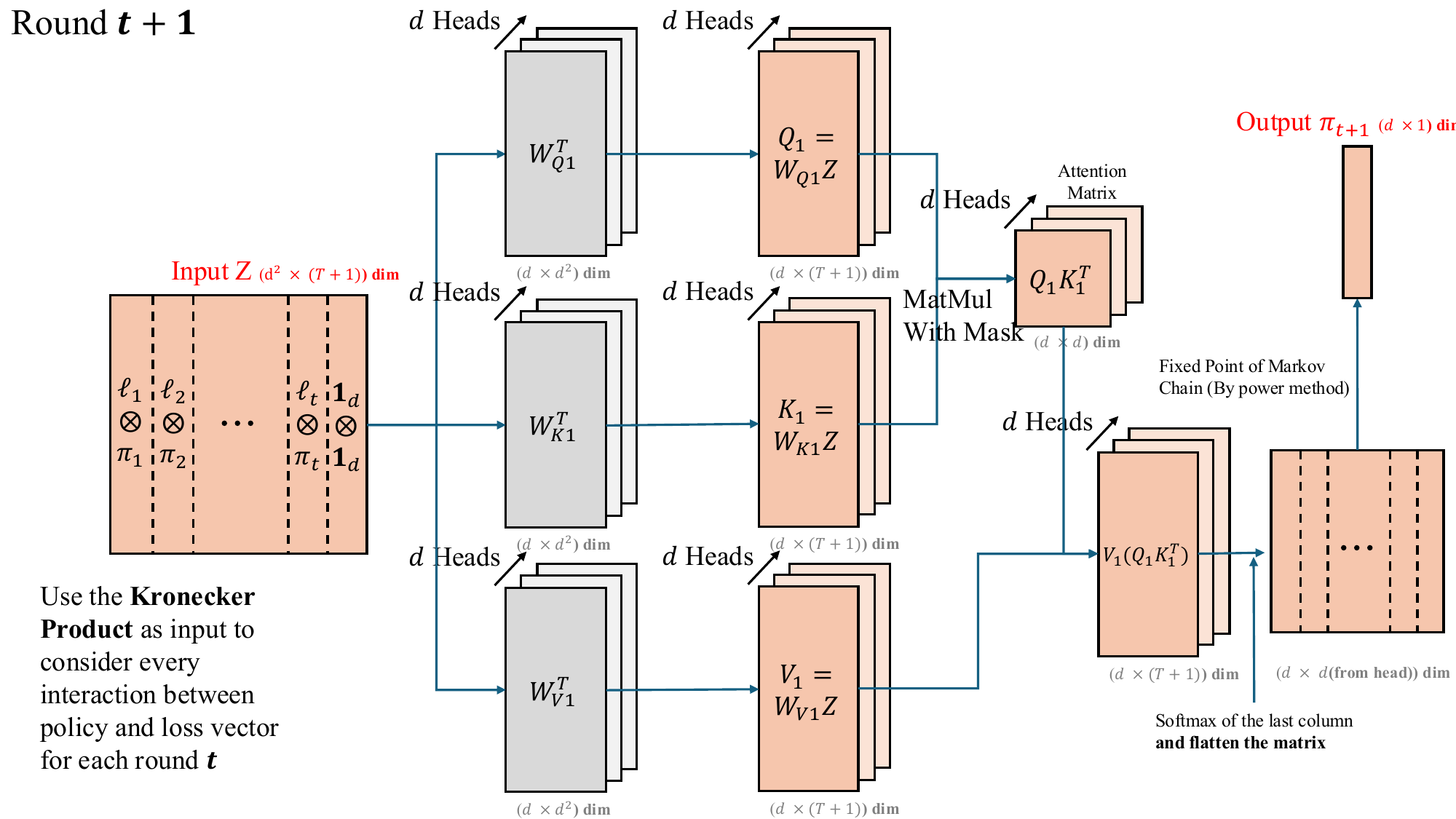}
  \caption{Multi-head linear self-attention architecture for swap-regret minimization. Each head processes the block features $X_t=\ell_t\otimes\pi_t$ and outputs one simplex vector $p_{j,t+1}$ through a $\softmax$ layer; these vectors form a transition matrix, and the final policy $\pi_{t+1}$ is obtained as its fixed point, matching the Blum--Mansour reduction from external regret to swap regret.}
  \label{fig:architecture}
\end{figure}

Now, we provide a theoretical analysis of the new architecture.
\begin{restatable}{theorem}{stationarynew}
  \label{thm:new-architecture-stationary}
  The configuration of the multi-head linear self-attention architecture 
  defined in \Cref{eqn:new-architecture} is a stationary point of the 
  swap-regret loss
  \begin{align}
  \mathcal{L}((V^{(j)}, a^{(j)}, v_{c}^{(j)})_{j\in[d]})
   &= \EE_{X\sim\mathcal{N}(0,I_{d^2T})}
      \left[
        \left(
          \min_{\hat{P}\in(\Delta([d]))^d}
          \sum_{t=1}^T \sum_{i=1}^d 
          \bigl\langle x_{i,t},\hat{P}_i - p_{i,t}\bigr\rangle
        \right)^2
      \right], \label{eq:Swap-regret-loss}
  \end{align}
  where 
  \[
  x_{i,t} := X_{d(i-1)+1:di,\; t} \in \mathbb{R}^d,
  \]
  and the parameters satisfy
  \[
  V^{(j)} = (V^{(j)}_1,\dots,V^{(j)}_d), \qquad 
  V^{(j)}_r =
  \begin{cases}
   -k I_{d\times d}, & r=j,\\[3pt]
   \pmb{O}_{d\times d}, & r\neq j,
  \end{cases}
  \qquad
  a^{(j)}=\pmb{0}_{d^2},\quad v_c^{(j)} = v\pmb{1}_d,
  \]
  for some $k=\Theta(1/\sqrt{T})$.
\end{restatable}
\paragraph{Proof structure for \Cref{thm:new-architecture-stationary}.}
At the proposed configuration, head \(j\) ignores all blocks of \(X_t\) except
the block \(x_{j,t}\), and its logits become
\[
    z_{j,t}=-k\sum_{s<t}x_{j,s}+(t-1)v\pmb{1}_d.
\]
Thus every head is a single-head smoothed fictitious-play learner, while the
swap objective couples the heads through the fixed-point policy and the
Blum--Mansour decomposition in \Cref{fact:BM-swap}.  The gradient calculation
separates this coupling carefully: shift invariance eliminates the
\(v_c^{(j)}\) and \(a^{(j)}\) directions, block independence eliminates
off-diagonal \(V^{(j)}_r\) terms for \(r\neq j\), and permutation equivariance
reduces the diagonal block \(V^{(j)}_j\) to the projector direction
\(I-d^{-1}\pmb{1}_d\pmb{1}_d^\intercal\).  The remaining directional
derivative along \(V^{(j)}_j=-kI\) is exactly \(d\) times the scalar derivative
of the single-head external-regret objective, so the same
\(k^\star\) scale from \Cref{thm:stationary} yields stationarity.

Here, \Cref{eq:Swap-regret-loss} is the square of the swap-regret decomposition in \Cref{eq:BM-swap-decomposition}. Therefore, \Cref{thm:new-architecture-stationary} shows that for the multi-head linear self-attention with $\FixedPoint$ operation, minimizing the swap regret loss provides a no-swap-regret guarantee, and the stationary point is equivalent to the Blum–Mansour swap-regret minimizer -- input and output are the same as the Blum–Mansour construction. Moreover, for each head, the configuration is equivalent to the external-regret smoothed fictitious play with appropriate stepsize. The proof is deferred to \Cref{appendix:pfthm3}.

We also provide an empirical $k^\star(d,T)$ consistent with this scaling, further supporting the theoretical prediction (Figure~\ref{fig:softmax-square}, right).

\begin{remark}[$\FixedPoint$ iteration for the swap-regret minimizer]
  A natural question is why the architecture requires a $\FixedPoint$ operation. The rationale is twofold. First, the computational complexity of fixed-point iteration is the same as that of classical no-swap-regret minimization \citep{hazan2007computational}, making it algorithmically consistent with known constructions. Second, most swap-regret and $\Phi$-regret algorithms explicitly rely on a fixed-point computation \citep{greenwald2003general}: a common approach is to minimize external regret in the $\Phi$-space and then recover the final policy via a fixed point of the induced transition operator. While alternative methods exist—such as the Markov-chain-tree–theorem–based approach of \citet{anagnostides2023near}, which replaces the fixed point with a linear map—that method requires $d^{\,d-1}$ parameters and is therefore exponentially large and unsuitable for transformer-based architectures. For these reasons, using a $\FixedPoint$ operation offers both theoretical alignment and practical scalability.
  \end{remark}

Finally, we provide a corollary that the stationary point of the new architecture is a CE in finite action games.
\begin{corollary}
  Consider a repeated game where each player deploys the trained multi-head single-layer linear attention from \Cref{thm:new-architecture-stationary} with \texttt{Operator} = $\softmax$. Then the time-averaged joint play converges to a CE in finite action games.
\end{corollary}
The distinction from the single-head result is the strength of the allowed deviations.  External-regret dynamics rule out fixed-action deviations and therefore imply CCE, whereas swap-regret dynamics rule out action-contingent deviations and therefore imply CE.  The fixed-point layer is the mechanism that converts the parallel external-regret heads into a single policy with swap-regret guarantees, so changing the regret surrogate changes the equilibrium concept approached by interacting agents.

\subsection{Correlated-equilibrium evaluation}
\label{ssec:empirical-ce}

We train all attention parameters of the \(d\)-head architecture jointly
from random initialization under the Gaussian block-input objective in
\Cref{eq:Swap-regret-loss}, then evaluate the learned heads through the
recurrent fixed-point layer.  Because the training objective uses full-scale
head blocks whereas the fixed-point recurrence supplies
\(x_{i,t}=\pi_{t,i}\ell_t\), we apply the predetermined dimension-only
calibration \(k_{\mathrm{head}}=d\,k_{\mathrm{train}}\).

For each run, we compute the standard maximum pairwise CE-deviation gap
directly from \(\bar\mu_T\).  At \(T=20\), the learned fixed-point
architecture obtains \(0.0627\pm0.0018\), a \(43.6\%\) reduction from
uniform play (\(0.1111\)) and nearly the value of the classical
Blum--Mansour entropy-FTRL algorithm (\(0.0617\)).  Its stronger full-swap
gap is \(0.0956\pm0.0068\).  Without the dimension calibration, the pairwise
and full-swap gaps are \(0.0968\pm0.0049\) and
\(0.1627\pm0.0110\), respectively.  Thus the finite-horizon empirical
distribution is an approximate, not exact, CE.

\begin{figure}[H]
  \centering
  \includegraphics[width=0.98\linewidth]{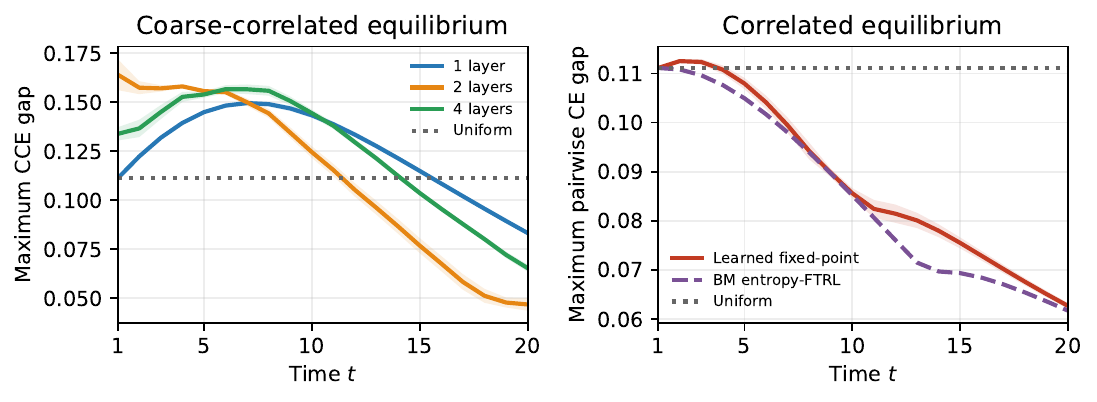}
  \caption{Equilibrium gaps of time-averaged joint mixed play in
  weighted rock--paper--scissors.  Left: external-regret-trained
  architectures and the maximum CCE gap.  Right: the learned fixed-point
  architecture, classical Blum--Mansour entropy-FTRL, and uniform play,
  evaluated by the maximum pairwise CE gap.  Curves are means with
  \(\pm1\) standard-error bands over the nine ordered pairs of three learned
  row and column checkpoints; algorithmic baselines are deterministic.}
  \label{fig:equilibrium-experiment}
\end{figure}

\section{Conclusion}
We revisited regret-based training of linear attention with policy space $\Delta(\cA)$ and showed that even minimal architectures admit interpretable stationary points once the non-linearity used at inference time is included in the training objective.  For single-head models on the simplex, the regret loss recovers smoothed fictitious play with stepsize $\tilde{\Theta}(1/\sqrt{T})$.  Extending the loss to swap regret and lifting the architecture to a multi-head configuration leads to a differentiable realization of the Blum–Mansour procedure in which each head implements an entropy-regularized FTRL learner and a $\FixedPoint$ layer aggregates the resulting transition matrix.

Our characterization highlights that regret supervision can endow attention models with provably no-regret behavior without auxiliary projections or post-processing.  Our analysis identifies interpretable stationary points rather than proving global optimality or convergence of training dynamics.  The proof also relies on full-information Gaussian training losses, simplex-valued policies, one-layer linear attention, the $\softmax$ operator, and horizon-dependent scaling of the learned stepsize.  The resulting forward pass, however, coincides with classical no-regret algorithms whose regret guarantees extend to bounded adversarial losses.

Several avenues remain open.  First, understanding whether deeper or non-linear attention layers share the same stationary geometry would bridge the gap to practical Transformers.  Second, translating these insights to LLMs, for example by pairing regret-based objectives with post-training pipelines as in \citet{park2025post}, would test whether the same mechanism appears at realistic scale.  Third, extending the stationary-point characterization to bounded non-Gaussian training distributions, horizon-free rates, and richer feedback models such as delayed or bandit observations would clarify how robust the structure is beyond the full-information Gaussian setting.

\clearpage
\bibliographystyle{ims}
\bibliography{main} 
\clearpage 
\appendix
\section{Proof of \Cref{thm:stationary}}
\label{appendix:pfthm1}
\stationary*
The proof will be divided into two parts. First, we prove the existence of the stationary point with $V = -k I_{d\times d}, a = \pmb{0}_d, v_c = v \pmb{1}_d$. Second, we show that $k$ will have a scale of $\Theta(1/\sqrt{T})$.
\begin{proof}[{\color{blue}\textbf{Proof for the existence of the stationary point with $V = -k I_{d\times d}, a = \pmb{0}_d, v_c = v \pmb{1}_d$}}]
    Fix $d\ge 2$ and horizon $T$. For $t\in[T]$ write
    \[
    S_t := \sum_{i=1}^t \ell_i \in \mathbb{R}^d,\qquad 
    z_t(V,a,v_c) \ :=\ \sum_{i=1}^t \big(V\,\ell_i\ell_i^\intercal a + (V+v_c a^\intercal)\ell_i + v_c\big)\in\mathbb{R}^d,
    \]
    and $\pi_t=\operatorname{\softmax}(z_t)$. With $V=-kI_{d\times d}$ and $a=\pmb{0}_d$ we have
    \[
    z_t(-k I_{d\times d},0,v_c) \ =\ -k\,S_t + t\,v_c.
    \]
    Regret can be written as:
    \[
    \mathrm{Regret}((\ell_t)_{t\le T}) \ :=\ \sum_{t=1}^T \langle \ell_t,\pi_t\rangle - \inf_{\pi\in\Delta^d}\,\Big\langle\sum_{t=1}^T \ell_t,\ \pi\Big\rangle
    \ =\ \sum_{t=1}^T \langle \ell_t,\pi_t\rangle - \min_{j\in[d]} (S_T)_j.
    \]
    For a fixed realization $(\ell_t)_{t\le T}$, write $R := \mathrm{Regret}((\ell_t)_{t\le T})$, and $\mathcal{L} \ =\ \mathbb{E}[h(R)]$, where $h(x)=x^2$. First, compute the derivative w.r.t.\ $\pi_t$. We have:
    \[
    \frac{\partial R}{\partial \pi_t} = \ell_t,
    \]
    since only the term $\langle \ell_t,\pi_t\rangle$ depends on $\pi_t$ and the comparator $\min_j (S_T)_j$ is constant with respect to the parameters.
    
    Let $J_{\mathrm{sm}}(z_t)\in\mathbb{R}^{d\times d}$ denote the Jacobian of $\softmax$ at $z_t$, i.e.
    \[
    J_{\mathrm{sm}}(z_t)
    \ =\ \frac{\partial\,\operatorname{\softmax}(z_t)}{\partial z_t}
    \ =\ \Diag(\pi_t) - \pi_t\pi_t^\intercal,\qquad \pi_t=\operatorname{\softmax}(z_t). 
    \]
    Then
    \[
    \frac{\partial R}{\partial z_t}
    = J_{\mathrm{sm}}(z_t)^\intercal\,\frac{\partial R}{\partial \pi_t}
    = J_{\mathrm{sm}}(z_t)^\intercal \ell_t.
    \]
    By the chain rule for $h(R)$,
    \[
    \frac{\partial}{\partial z_t} h(R)\ =\ h'(R)\,\frac{\partial R}{\partial z_t}
    = 2R\,J_{\mathrm{sm}}(z_t)^\intercal \ell_t.
    \]
    
    We define the (random) vector
    \[
    \Delta_t\ :=\ \frac{\partial}{\partial z_t} h(R)
    \ =\ 2\,\mathrm{Regret}((\ell_s)_{s\le T})\;J_{\mathrm{sm}}(z_t)^\intercal \ell_t
    \ \in\ \mathbb{R}^d.
    \]
    The mapping $(z_1,\dots,z_T)\mapsto h(R)$ is $C^1$, and by dominated convergence ($\softmax$ and $h$ are smooth, Gaussian tails), we may exchange gradient and expectation, so for any perturbation $(\delta z_t)_{t\le T}$:
    \[
    \delta \mathcal{L}\ =\ \mathbb{E}\Big[\sum_{t=1}^T \Big\langle \Delta_t,\ \delta z_t\Big\rangle \Big].
    \]
    
    For arbitrary perturbations $\delta V,\delta a,\delta v$ (with $\delta v_c=\delta v\,\pmb{1}_d$), $\delta z_t$ is given by
    \begin{align*}
    \delta z_t
    &= \sum_{i=1}^t \Big(\delta V\,\ell_i\ell_i^\intercal a + V\,\ell_i\ell_i^\intercal \delta a\Big)
     + \sum_{i=1}^t \Big(\delta V\,\ell_i + v_c\,\langle \delta a,\ell_i\rangle + \delta v_c\,a^\intercal \ell_i\Big)
     + \sum_{i=1}^t \delta v_c\\
    &= \underbrace{\sum_{i=1}^t \big(\delta V\,\ell_i\ell_i^\intercal a + \delta V\,\ell_i\big)}_{(\ast_V)}
    \ +\ \underbrace{\sum_{i=1}^t \big(V\,\ell_i\ell_i^\intercal \delta a + v_c\,\ell_i^\intercal \delta a\big)}_{(\ast_a)}
    \ +\ \underbrace{\sum_{i=1}^t \big(\delta v_c\,a^\intercal\ell_i + \delta v_c\big)}_{(\ast_v)}.
    \end{align*}
    
    We now instantiate at $(V,a,v_c)=(-k I_{d\times d},0,v\pmb{1}_d)$, so that $S_t=\sum_{i=1}^t \ell_i$ and
    \[
    (\ast_V) = \delta V\,S_t,\qquad
    (\ast_a) = \big(-k\sum_{i=1}^t \ell_i\ell_i^\intercal + v\,\pmb{1}_d S_t^\intercal\big)\delta a,\qquad
    (\ast_v) = t\,\delta v\,\pmb{1}_d.
    \]
    Thus at $(V,a,v_c)=(-k I_{d\times d},0,v\pmb{1}_d)$ we may write
    \begin{equation}\label{eq:master}
    \delta \mathcal{L}\ =\ \mathbb{E}\Big[\sum_{t=1}^T \langle \Delta_t,\ \delta z_t\rangle \Big],
    \quad
    \delta z_t\ =\ \delta V\,S_t\ +\ \big(-k\sum_{i=1}^t \ell_i\ell_i^\intercal + v\,\pmb{1}_d S_t^\intercal\big)\delta a\ +\ t\,\delta v\,\pmb{1}_d,
    \end{equation}
    with
    \[
    \Delta_t\ =\ 2\,\mathrm{Regret}((\ell_s)_{s\le T})\;\big(\Diag(\pi_t)-\pi_t\pi_t^\intercal\big)^\intercal \ell_t,
    \quad
    \pi_t=\operatorname{\softmax}(z_t).
    \]
    
    \paragraph{(i) The $v$-direction.}
    Using $\delta z_t^{(v)} = t\,\delta v\,\pmb{1}_d$ in \eqref{eq:master} we obtain
    \[
    \delta \mathcal L^{(v)}
    = \delta v\,\sum_{t=1}^T t\,\mathbb E\big[\langle \Delta_t,\pmb{1}_d\rangle\big].
    \]
    Now
    \[
    \Delta_t
    = 2R\big(\Diag(\pi_t)-\pi_t\pi_t^\intercal\big)\ell_t,\qquad \pi_t=\operatorname{\softmax}(z_t),
    \]
    and hence
    \begin{align*}
    \langle \Delta_t,\pmb{1}_d\rangle
    &= 2R\,\pmb{1}_d^\intercal\big(\Diag(\pi_t)-\pi_t\pi_t^\intercal\big)\ell_t\\
    &= 2R\,\big(\underbrace{\pmb{1}_d^\intercal \Diag(\pi_t)}_{=\pmb{1}_d^\intercal\pi_t=1}
              -\underbrace{\pmb{1}_d^\intercal \pi_t\pi_t^\intercal}_{=1\cdot\pi_t^\intercal}\big)\ell_t
    = 2R\,(1-1)\,\pi_t^\intercal \ell_t
    = 0.
    \end{align*}
    Thus each term in the sum vanishes pointwise, so $\delta\mathcal L^{(v)}=0$ for all $\delta v$, i.e.
    \[
    \frac{\partial\mathcal L}{\partial v}=0,
    \]
    which matches (and refines) the earlier shift-invariance argument.
    
    \paragraph{(ii) The $a$-direction at $a=\pmb{0}_d$.}
    At $(V,a,v_c)=(-k I_{d\times d},0,v\pmb{1}_d)$, \eqref{eq:master} gives
    \begin{align*}
    \delta\mathcal L^{(a)}
    &= \mathbb E\Big[\sum_{t=1}^T
      \Big\langle \Delta_t,\ \big(-k\sum_{i=1}^t \ell_i\ell_i^\intercal + v\,\pmb{1}_d S_t^\intercal\big)\delta a\Big\rangle\Big]\\
    &= -k\,\mathbb E\Big[\sum_{t=1}^T\Big\langle \Delta_t,\ \sum_{i=1}^t \ell_i\ell_i^\intercal \delta a\Big\rangle\Big]
    \ +\ v\,\mathbb E\Big[\sum_{t=1}^T \langle \Delta_t,\pmb{1}_d S_t^\intercal\delta a\rangle\Big].
    \end{align*}
    The second term vanishes because $\pmb{1}_d^\intercal \Delta_t = 0$:
    \[
    \langle \Delta_t,\pmb{1}_d S_t^\intercal\delta a\rangle
    = (\pmb{1}_d^\intercal \Delta_t)\,\langle S_t,\delta a\rangle
    = 0.
    \]
    Hence
    \[
    \delta\mathcal L^{(a)}
    = -k\,\Big\langle \mathbb E\Big[\sum_{t=1}^T\sum_{i=1}^t \ell_i\ell_i^\intercal \Delta_t\Big],\ \delta a\Big\rangle.
    \]
    We may write
    \[
    \nabla_a \mathcal L\big|_{(V,a,v_c)=(-k I_{d\times d},0,v\pmb{1}_d)}
    = -k\,\mathbb E\Big[\sum_{t=1}^T\sum_{i=1}^t \ell_i\ell_i^\intercal \Delta_t\Big]\ \in\ \mathbb R^d.
    \]
    
    Set
    \[
    H := \mathbb{E}\Big[\sum_{t=1}^T \sum_{i=1}^t \ell_i\ell_i^\intercal \Delta_t\Big]\in\mathbb{R}^d,
    \]
    so that $\nabla_a \mathcal L = -k H$. We claim $H=\pmb{0}_d$.
    
    To see this, use the \emph{permutation invariance} of the law of $(\ell_t)_{t\le T}$ (it is spherically symmetric since $\ell_t \sim \mathcal{N}(0, I_{d\times d})$ i.i.d.) and the \emph{permutation equivariance}
    of $\softmax$: for any permutation matrix $P$, $P\,\operatorname{\softmax}(x)=\operatorname{\softmax}(Px)$.
    Under $(\ell_t)\mapsto (P\ell_t)$ we have $S_t\mapsto PS_t$ and hence $z_t\mapsto -k\,PS_t$, so
    $\pi_t\mapsto P\pi_t$ and $\Delta_t\mapsto P\Delta_t$. Therefore
    \[
    H\ =\ \mathbb{E}\Big[\sum_{t,i} (P\ell_i)(P\ell_i)^\intercal (P\Delta_t)\Big]
    \ =\ P\,\mathbb{E}\Big[\sum_{t,i}\ell_i\ell_i^\intercal \Delta_t\Big]\ =\ P\,H\quad\text{for all permutations }P.
    \]
    The only vectors fixed by all coordinate permutations are multiples of $\pmb{1}_d$. On the other hand,
    \[
    \pmb{1}_d^\intercal H
    = \mathbb{E}\Big[\sum_{t,i} \langle \pmb{1}_d,\ \ell_i\ell_i^\intercal \Delta_t\rangle\Big]
    = \mathbb{E}\Big[\sum_{t,i} \langle \ell_i, \Delta_t\rangle\cdot\langle \pmb{1}_d,\ell_i\rangle\Big].
    \]
    By symmetry the joint law is invariant under $\ell_i\mapsto -\ell_i$ while $\Delta_t$ changes sign exactly as
    $\ell_t$ does through the $\softmax$ Jacobian (the mapping is odd along each coordinate in expectation); thus
    $\pmb{1}_d^\intercal H=0$. Hence the only permutation-fixed vector orthogonal to $\pmb{1}_d$ is $\pmb{0}_d$,
    so $H=\pmb{0}_d$ and therefore
    \[
    \big.\nabla_a \mathcal{L}\big|_{a=\pmb{0}}=\pmb{0}_d.
    \]
    
    \paragraph{(iii) The $V$-direction at $V=-k I_{d\times d}$.}
    From $(\ast_V)$, at $a=0$ we have $\delta z_t = \delta V\,S_t$. Using \eqref{eq:master} and writing the Frobenius product,
    \begin{align*}
    \delta \mathcal{L}
    &= \mathbb{E}\Big[\sum_{t=1}^T \langle \Delta_t,\ \delta V S_t\rangle\Big]
    = \mathbb{E}\Big[\sum_{t=1}^T \langle \Delta_t S_t^\intercal,\ \delta V\rangle_F\Big]\\
    &= \Big\langle \underbrace{\mathbb{E}\Big[\sum_{t=1}^T \Delta_t S_t^\intercal\Big]}_{=:G(V)},\ \delta V\Big\rangle_F,
    \end{align*}
    so $\nabla_V \mathcal{L}=G(V)$. Evaluate at $V=-k I_{d\times d}$. For every
    permutation matrix $P$, the Gaussian law is invariant under
    $(\ell_t)\mapsto(P\ell_t)$, and $\softmax$ is permutation-equivariant.  Thus
    $(S_t,\Delta_t)\stackrel{d}{=}(PS_t,P\Delta_t)$ and
    \[
    G(-k I_{d\times d})\ =\ \mathbb{E}\Big[\sum_{t=1}^T \Delta_t S_t^\intercal\Big]
    \ =\ \mathbb{E}\Big[\sum_{t=1}^T (P\Delta_t)(PS_t)^\intercal\Big]
    \ =\ P\,G(-k I_{d\times d})\,P^\intercal
    \]
    for all permutation matrices $P$. By \Cref{lem:permutation-commutant}, there
    are scalars $\alpha(k),\beta(k)$ such that
    \[
    G(-k I_{d\times d})=\alpha(k)I_{d\times d}
    +\beta(k)\pmb{1}_d\pmb{1}_d^\intercal.
    \]
    Moreover, $\pmb{1}_d^\intercal\Delta_t=0$ for every $t$ because
    $J_{\mathrm{sm}}(z_t)^\intercal\pmb{1}_d=\pmb{0}_d$, and hence
    $\pmb{1}_d^\intercal G(-k I_{d\times d})=\pmb{0}_d^\intercal$. Therefore
    \[
    \nabla_V \mathcal{L}(-k I_{d\times d},0,v\pmb{1}_d)
    =\alpha(k)\Pi_0,\qquad
    \Pi_0:=I_{d\times d}-\frac{1}{d}\pmb{1}_d\pmb{1}_d^\intercal.
    \]
    
    Let $k^\star$ be any stationary point of the one-dimensional function 
    \[
    \phi(k):=\mathcal{L}(V=-k I_{d\times d},a=0,v_c=v\pmb{1}_d),
    \]
    which exists by smoothness and standard compactification/continuity
    arguments; at a local minimizer we certainly have $\phi'(k^\star)=0$. Then
    \[
    0\ =\ \phi'(k^\star)\ =\ \Big\langle \nabla_V \mathcal{L}(-k^\star I,0,v\pmb{1}_d),\ \frac{\partial(-k I_{d\times d})}{\partial k}\Big\rangle_F
    \ =\ \Big\langle \alpha(k^\star)\Pi_0,\ -I\Big\rangle_F \ =\ -(d-1)\alpha(k^\star),
    \]
    so $\alpha(k^\star)=0$ and hence $\nabla_V \mathcal{L}(-k^\star I,0,v\pmb{1}_d)=\pmb{0}$.
    
    \paragraph{Conclusion.}
    At the parameter triple $(V,a,v_c)=(-k^\star I_{d\times d},\ \pmb{0}_d,\ v\,\pmb{1}_d)$ we have
    \[
    \nabla_V \mathcal{L}=\pmb{O}_{d\times d},\qquad \nabla_a \mathcal{L}=\pmb{0}_d,\qquad \nabla_{v_c} \mathcal{L}=\pmb{0}_d.
    \]
    Thus the configuration is a first-order stationary point of $\mathcal{L}$.
    \end{proof}
    
    \begin{remark}
    The argument uses only: (1) $\softmax$ shift-invariance and its Jacobian $J_{\mathrm{sm}}(z)=\Diag(\pi)-\pi\pi^\intercal$;
    (2) permutation invariance of the Gaussian hierarchical model $\ell_t \sim \mathcal{N}(0, I_{d\times d})$ i.i.d.; and (3) the zero-sum range of the $\softmax$ Jacobian.
    \end{remark}
    Now we show that $k$ will have a scale of $\Theta(1/\sqrt{T})$.
\begin{proof}[{\color{blue}\textbf{Proof for the scale of $k$}}]

        Fix integers $d\ge 2$ and $T\ge 1$. Let $(\ell_t)_{t=1}^\intercal$ be i.i.d.\ $\mathcal N(0,I_{d\times d})$, and define partial sums
        $S_t:=\sum_{i=1}^t \ell_i$ with $S_0:=0$, exactly as in the previous part of the proof.
        
        For $k\in\mathbb R$, the $\softmax$ policy induced by $V=-kI_{d\times d}$, $a=\pmb{0}_d$, $v_c=v\pmb{1}_d$ reads
        \[
        \pi_t(k) := \operatorname{\softmax}(-k S_{t-1})\in\Delta_d,
        \qquad
        J_t(k):=J_{\mathrm{sm}}(-k S_{t-1})=\Diag(\pi_t(k))-\pi_t(k)\pi_t(k)^\intercal,
        \]
        where we re-use the Jacobian notation $J_{\mathrm{sm}}$ from the first part.
        Let the regret against the best fixed coordinate be
        \[
        R_T(k):=\sum_{t=1}^T \langle \ell_t,\pi_t(k)\rangle - \min_{j\in[d]} S_{T,j},
        \]
        and the objective $\mathcal L(k):=\mathbb E[R_T(k)^2]$.
        Let $J^\star:=\arg\min_{j\in[d]} S_{T,j}$ (unique a.s.).
        Write $u:=\tfrac1d\pmb{1}_d$ and $J_0:=\Diag(u)-uu^\intercal=\tfrac1d \Pperp{\pmb{1}_d}$.
        
        We will use the following two elementary facts.

    \begin{fact}\label{fact:conditional-expectation}
        For each $t\le T$, $\EE[S_{t-1}\mid S_T]=\frac{t-1}{T}S_T$ and 
        $S_{t-1}-\frac{t-1}{T}S_T \perp\!\!\!\perp S_T$.
        Hence
        \[
        \EE\big[S_{t-1,J^\star}\big]
        =\frac{t-1}{T}\,\EE\!\left[\min_{1\le j\le d} S_{T,j}\right].
        \]
        \end{fact}

        \begin{fact}[Exact derivative identity]\label{lem:derivative}
        For every $k\in\mathbb R$, 
        \begin{equation}\label{eq:split}
            \mathcal L'(k)
            = -2\underbrace{\sum_{t=1}^T \mathbb E\!\big[\langle J_t(k)S_{t-1},\,\pi_t(k)-u\rangle\big]}_{=:\Sigma_1}
            \;+\; 2\underbrace{\sum_{t=1}^T \mathbb E\!\big[\langle J_t(k)S_{t-1},\,e_{J^\star}\rangle\big]}_{=:\Sigma_2}.
            \end{equation}             
        \end{fact}

        \begin{proof}
        Since the comparator $\min_j S_{T,j}$ is $k$-independent,
        \[
        \frac{d}{dk}\sum_{t=1}^T \langle \ell_t,\pi_t(k)\rangle
        = -\sum_{t=1}^T \langle \ell_t, J_t(k) S_{t-1}\rangle.
        \]
        Thus $\mathcal L'(k)=2\,\mathbb E\!\big[R_T(k)\cdot \tfrac{d}{dk}\sum_t \langle \ell_t,\pi_t(k)\rangle\big]
        = -2\sum_t \mathbb E\!\big[R_T(k)\langle \ell_t, J_t(k) S_{t-1}\rangle\big]$.
        
        Fix $t$ and condition on $\mathcal F_{t-1}=\sigma(\ell_1,\dots,\ell_{t-1})$.
        Stein’s lemma (\Cref{lem:stein}) for $\ell_t\sim\mathcal N(0,I_{d\times d})$ yields
        \[
        \mathbb E\!\big[\langle \ell_t,v\rangle R_T(k)\mid\mathcal F_{t-1}\big]
        =\mathbb E\!\big[\langle v,\nabla_{\ell_t}R_T(k)\rangle\mid\mathcal F_{t-1}\big],
        \quad v=J_t(k) S_{t-1}.
        \]
        Because $\partial_{\ell_t}\langle \ell_t,\pi_t(k)\rangle=\pi_t(k)$, for $s>t$ one has
        $\partial_{\ell_t}\langle \ell_s,\pi_s(k)\rangle=-k\,J_s(k) \ell_s$, and
        $\partial_{\ell_t}\min_j S_{T,j}=e_{J^\star}$ almost surely; hence
        \(
        \nabla_{\ell_t}R_T(k)=\pi_t(k) - e_{J^\star} - k\sum_{s>t} J_s(k) \ell_s.
        \)
        Therefore
        \[
        \mathbb E\!\big[R_T(k)\langle \ell_t,v\rangle\big]
        =\mathbb E[\langle v,\pi_t(k)-e_{J^\star}\rangle]
        - k\sum_{s>t}\mathbb E[\langle v, J_s(k)\ell_s\rangle].
        \]
        Conditioning on $\mathcal F_{s-1}$, the second term is $0$ because $\mathbb E[\ell_s\mid \mathcal F_{s-1}]=0$.
        Summing in $t$ gives the claim. Here, we also used that $J_s(k)\pmb{1}_d = \pmb{0}_d$ for all $s$ and $k$ by definition of $J_s(k)$.
        \end{proof}

Now, we provide a first-order expansion of the policy with bounds on the remainder term.
\begin{proposition}[First–order expansion of the policy]\label{prop:policy-expansion}
    The following equation and inequalities hold:
    \begin{equation}\label{eq:policy-expansion}
    \pi_t(k) - u \;=\; -\,J_0\,k\,S_{t-1} \;+\; r_t,\qquad
    \|r_t\|_2 \le \frac{3}{4}\,k^2\,\|S_{t-1}\|_2^2, \qquad 
    \|r_t\|_\infty \le k^2\,\|S_{t-1}\|_\infty^2.
    \end{equation}
    \end{proposition}
    
    \begin{proof}
    Use the fundamental theorem of calculus on the ray $\theta\mapsto \theta z_t$:
    \[
    \pi_t(k)-u \;=\; \operatorname{\softmax}(-kS_{t-1})-\operatorname{\softmax}(0) \;=\;  \int_0^1 J_{\mathrm{sm}}(\theta (-kS_{t-1}))\,(-kS_{t-1})\,d\theta.
    \]
    Add and subtract $- J_0 kS_{t-1}$ and define
    \(
    r_t:=\int_0^1 \big(J_{\mathrm{sm}}(\theta (-kS_{t-1}))-J_0\big)\,(-kS_{t-1})\,d\theta,
    \)
    Then by \Cref{lem:softmax}-(d),
    \[
    \|r_t\|_2 \;\le\; \int_0^1 \big\|J_{\mathrm{sm}}(\theta (-kS_{t-1}))-J_0\big\|_{\mathrm{op}}\,\|(-kS_{t-1})\|_2\,d\theta
    \;\le\; \int_0^1 \frac{3}{2}\theta\|kS_{t-1}\|_2^2\,d\theta
    = \frac{3}{4}\|kS_{t-1}\|_2^2.
    \]
    Thus \eqref{eq:policy-expansion} holds with $C_1:=\frac{3}{4}$.
    Moreover, using \Cref{lem:softmax}-(e) with $h=-kS_{t-1}$,
    \[
    \bigl\|J_{\sm}(\theta (-kS_{t-1}))-J_0\bigr\|_{\infty\to\infty}
    \le 2\theta\,|k|\,\|S_{t-1}\|_\infty.
    \]
    Therefore, the following inequality holds:
    \[
    \|r_t\|_\infty
    \le \int_0^1 \bigl\|J_{\sm}(-\theta kS_{t-1})-J_0\bigr\|_{\infty\to\infty}\,\|kS_{t-1}\|_\infty\,d\theta
    \le \int_0^1 2\theta\,|k|\,\|S_{t-1}\|_\infty\,|k|\,\|S_{t-1}\|_\infty\,d\theta
    = k^2\|S_{t-1}\|_\infty^2.
    \]
    \end{proof}

    \begin{proposition}[First–order expansion of the Jacobian]\label{prop:jac-expansion}
        There exists an absolute $L>0$ such that
        \begin{equation}\label{eq:jac-expansion}
        J_t(k) \;=\; J_0 \;+\; R_t,
        \qquad
        \|R_t\|_{\mathrm{op}} \;\le\; L\,k\,\|S_{t-1}\|_2 \qquad \norm{R_t}_{\infty \to \infty} \le 2 k \|S_{t-1}\|_\infty.
        \end{equation}
        \end{proposition}
        \Cref{prop:jac-expansion} is a re-instantitation of \Cref{lem:softmax}-(e). 

    We now establish two auxiliary lemmas specific to this proof.

\begin{theorem}[Negativity for $k\le c/\sqrt{T d}$] \label{thm:smallk}
    There exists an absolute constant $c_\ast>0$ such that if $k = \tilde{\mathcal{O}}(1/\sqrt{T d})$, then $\mathcal L'(k)<0$.
    \end{theorem}
\begin{proof}
    First, we provide an upper bound on $-2\Sigma_1$. Fix $t$ and abbreviate $S:=S_{t-1}$, $\pi:=\pi_t(k)$, $J_t:=J_t(k)$.
    From \Cref{prop:policy-expansion} and \Cref{prop:jac-expansion}, we have
    \[
    \pi-u = -kJ_0 S + r_t,
    \qquad
    J_t = J_0 + R_t.
    \]
    By plugging in formula for $\Sigma_1$, we have
    \begin{align}
    \langle J_t S,\pi-u\rangle
    &= \langle (J_0+R_t)S, -kJ_0S + r_t\rangle \notag\\
    &= -k\|J_0S\|_2^2
    \;+\; \underbrace{\langle J_0S,r_t\rangle}_{E_{1,t}}
    \;-\; \underbrace{k\langle R_tS,J_0S\rangle}_{E_{2,t}}
    \;+\; \underbrace{\langle R_tS,r_t\rangle}_{E_{3,t}}.
    \label{eq:Sigma1-E123}
    \end{align}
    Therefore, we have 
    \begin{equation}\label{eq:Sigma1-upper-structure}
    -2\Sigma_1
    \;\le\;
    2k\sum_{t=1}^T \EE\|J_0S_{t-1}\|_2^2
    + 2\sum_t \EE|E_{1,t}|
    + 2k\sum_t \EE|E_{2,t}|
    + 2\sum_t \EE|E_{3,t}|.
    \end{equation}
    
    {\color{red} \emph{Main term.}}
    Because $S_{t-1}\sim\cN(\pmb{0}_d,(t-1)I_{d\times d})$ and $J_0$ is a projection (scaled by $1/d$) onto an $(d-1)$-dimensional subspace,
    \[
    \EE\|J_0S_{t-1}\|_2^2 = \frac{d-1}{d^2}(t-1),
    \]
    and thus
    \begin{equation}\label{eq:Sigma1-main}
    2k\sum_{t=1}^T \EE\|J_0S_{t-1}\|_2^2
    = k\frac{d-1}{d^2}T(T-1).
    \end{equation}
    
    {\color{red} \emph{Bound on $E_{1,t}=\langle J_0S,r_t\rangle$.}}
    By Hölder inequality in $(1,\infty)$, \Cref{prop:policy-expansion}, and $\|S_{t-1}\|_\infty \le \|S_{t-1}\|_\infty$,
    \[
    |E_{1,t}|
    \le \|J_0S\|_1\|r_t\|_\infty
    \le \|J_0\|_{1\to 1}\|S\|_1\cdot k^2\|S\|_\infty^2
    \le \frac{2}{d} k^2 \|S\|_1 \|S\|_\infty^2
    \le 2k^2 \|S\|_\infty^3. 
    \]
    Taking expectations and summing over $t$ with \Cref{lem:gauss-max-final}, for some absolute constant $C_1>0$,
    \begin{equation}\label{eq:E1t-sum}
    \sum_{t=1}^T \EE|E_{1,t}|
    \le 2 k^2 \sum_{t=1}^T \EE\|S_{t-1}\|_\infty^3 \le C_1 k^2 T^{5/2}(\log(e d))^{3/2}.
    \end{equation}

    {\color{red} \emph{Bound on $E_{2,t}=k\langle R_tS,J_0S\rangle$.}}
We will prove there exists an absolute constant $C>0$ such that for all $t\le T$,
\begin{equation}\label{eq:E2-single-step}
\big|\EE[E_{2,t}]\big|
\;\le\;
\frac{C}{d^{3/2}}\,
k^2\,(t-1)^{3/2}\big(\log(ed)\big)^{3/2}\,e^{4k^2(t-1)}.
\end{equation}
and consequently,
\begin{equation}\label{eq:E2-sum}
\sum_{t=1}^T \big|\EE[E_{2,t}]\big|
\;\le\;
\frac{C}{d^{3/2}}\,
k^2\,T^{5/2}\big(\log(ed)\big)^{3/2}\,e^{4k^2T}.
\end{equation}
By Lemma~\ref{lem:null-direction}, $(J_t(k)-J_0)S = (J_t(k)-J_0)J_0S$, so
\[
E_{2,t}
= k\big\langle (J_{\sm}(-kS)-J_0)J_0S,\ J_0S\big\rangle.
\]

Let $J(z) := J_{\sm}(z)$.
Along the ray $\theta\mapsto -\theta kS$ we have
\[
J(-kS)-J_0 \;=\; \int_0^1 DJ(-\theta kS)[-kS]\,d\theta.
\]
Therefore
\begin{equation}\label{eq:E2-ray-representation}
E_{2,t}
= k\int_0^1 \big\langle DJ(-\theta kS)[-kS]\,J_0S,\ J_0S\big\rangle\,d\theta.
\end{equation}
Expand $DJ(-\theta kS)$ around $0$:
\[
DJ(-\theta kS)[-kS]
= DJ(0)[-kS] \;+\; \int_0^\theta D^2J(-\tau kS)[-kS,-kS]\,d\tau.
\]
Plugging this into \eqref{eq:E2-ray-representation}, we obtain
\[
E_{2,t} = H_{\mathrm{lin}}(S) + H_{\mathrm{rem}}(S),
\]
where
\begin{align*}
H_{\mathrm{lin}}(S)
&:= k\big\langle DJ(0)[-kS]\,J_0S,\ J_0S\big\rangle,\\
H_{\mathrm{rem}}(S)
&:= k\int_0^1\!\int_0^\theta
\big\langle D^2J(-\tau kS)[-kS,-kS]\,J_0S,\ J_0S\big\rangle\,d\tau\,d\theta.
\end{align*}

Consider the transformation $S\mapsto -S$.
Since $J_0$ is linear, $J_0(-S)=-J_0S$.
Furthermore, $DJ(0)[h]$ is linear in $h$, so
$DJ(0)[-k(-S)] = DJ(0)[kS] = -DJ(0)[-kS]$.
A direct sign check yields
\[
H_{\mathrm{lin}}(-S)
= k\big\langle DJ(0)[-k(-S)]\,J_0(-S),\ J_0(-S)\big\rangle
= -H_{\mathrm{lin}}(S).
\]
Since $S$ is symmetric in law, $\cL(S)=\cL(-S)$, we obtain
$\EE[H_{\mathrm{lin}}(S)]=0$ and therefore
\[
\EE[E_{2,t}] = \EE[H_{\mathrm{rem}}(S)].
\]

Let
\[
M_\tau(S) := D^2J(-\tau kS)[-kS,-kS].
\]
By Hölder in $(\infty,1)$,
\[
\big|\langle M_\tau(S)J_0S,\ J_0S\rangle\big|
\le \|M_\tau(S)\|_{\infty\to\infty}\,\|J_0S\|_\infty\,\|J_0S\|_1.
\]
By Lemma~\ref{lem:softmax}\,(f),
\[
\|M_\tau(S)\|_{\infty\to\infty}
\le \frac{16}{d}\,e^{2\tau kq}\,k^2 q^2,
\qquad q:=\|S\|_\infty.
\]
By Lemma~\ref{lem:J0-infty},
\[
\|J_0S\|_\infty\le \frac{2}{d}\,q,
\qquad
\|J_0S\|_1\le 2\,q.
\]
Hence
\[
\big|\langle M_\tau(S)J_0S,\ J_0S\rangle\big|
\le \frac{64}{d^2}\,k^2 q^4 e^{2\tau kq}.
\]

Therefore,
\begin{align*}
|H_{\mathrm{rem}}(S)|
&\le k\int_0^1\!\int_0^\theta
\big|\langle M_\tau(S)J_0S,\ J_0S\rangle\big|\,d\tau\,d\theta\\
&\le \frac{64}{d^2}\,k^3 q^4\int_0^1\!\int_0^\theta e^{2\tau kq}\,d\tau\,d\theta.
\end{align*}
Using
\[
\int_0^1\!\int_0^\theta e^{2\tau kq}\,d\tau\,d\theta
= \int_0^1 (1-\tau)e^{2\tau kq}\,d\tau
\le \int_0^1 e^{2\tau kq}\,d\tau
= \frac{e^{2kq}-1}{2kq}
\le \frac{1}{2kq}e^{2kq},
\]
we obtain
\[
|H_{\mathrm{rem}}(S)|
\le \frac{32}{d^2}\,k^2 q^3 e^{2kq}.
\]
Taking expectations,
\[
\big|\EE[E_{2,t}]\big|
= \big|\EE[H_{\mathrm{rem}}(S)]\big|
\le \frac{32}{d^2}\,k^2\,\EE\big[q^3 e^{2kq}\big].
\]

Apply Cauchy--Schwarz:
\[
\EE\big[q^3 e^{2kq}\big]
\le \big(\EE[q^6]\big)^{1/2}\big(\EE[e^{4kq}]\big)^{1/2}.
\]
By Lemma~\ref{lem:gauss-max-final} with $q=6$,
\[
\EE[q^6]
\lesssim (t-1)^3\big(\log(ed)\big)^3.
\]
By the MGF bound for the sup (Lemma~\ref{lem:mgf-sup}),
\[
\EE[e^{4kq}]
\lesssim d\,\exp\big(8k^2(t-1)\big).
\]
Thus
\[
\EE\big[q^3 e^{2kq}\big]
\lesssim (t-1)^{3/2}\big(\log(ed)\big)^{3/2}\,d^{1/2}e^{4k^2(t-1)}.
\]

Combining this with the prefactor $32/d^2$ gives
\[
\big|\EE[E_{2,t}]\big|
\;\le\;
\frac{C}{d^{3/2}}\,
k^2\,(t-1)^{3/2}\big(\log(ed)\big)^{3/2}e^{4k^2(t-1)}
\]
for some absolute constant $C>0$, proving \eqref{eq:E2-single-step}.
Summing over $t$ and using $\sum_{t=1}^T (t-1)^{3/2}\lesssim T^{5/2}$ yields \eqref{eq:E2-sum}.

    {\color{red} \emph{Bound on $E_{3,t}=\langle R_tS,r_t\rangle$.}}
    We use the integral representation of $r_t$.
    For $\theta\in[0,1]$ set
    \(
    A_\theta:=J_{\sm}(\theta(-kS))-J_0
    \)
    so that $R_t=A_1$ and
    \(
    r_t=\int_0^1 A_\theta (-kS)\,d\theta=-k\int_0^1 A_\theta S\,d\theta.
    \)
    By \Cref{lem:null-direction}, $A_\theta S=A_\theta J_0S$ and $R_tS=R_tJ_0S$. Therefore
    \[
    E_{3,t}
    = \big\langle R_tS,\,r_t\big\rangle
    = -k\int_0^1 \big\langle R_t(J_0S),\,A_\theta(J_0S)\big\rangle\,d\theta.
    \]
    Taking absolute values and using Hölder in $(\infty,1)$,
    \[
    |E_{3,t}|
    \le k\int_0^1 \|R_t(J_0S)\|_\infty\,\|A_\theta(J_0S)\|_1\,d\theta
    \le k\int_0^1 \|R_t\|_{\infty\to\infty}\,\|J_0S\|_\infty\cdot \|A_\theta\|_{1\to 1}\,\|J_0S\|_1\,d\theta.
    \]
    By \Cref{lem:softmax}-(e) with $h=-kS$,
    \(
    \|R_t\|_{\infty\to\infty}\le 2k\|S\|_\infty
    \)
    and
    \(
    \|A_\theta\|_{1\to 1}=\|J_{\sm}(\theta(-kS))-J_0\|_{1\to 1}\le 2\theta k\|S\|_\infty.
    \)
    By the definition of $J_0$,
    \(
    \|J_0S\|_\infty\le \tfrac{2}{d}\|S\|_\infty
    \)
    and
    \(
    \|J_0S\|_1\le \tfrac{2}{d}\|S\|_1\le 2\|S\|_\infty
    \)
    (the last inequality uses $\|S\|_1\le d\|S\|_\infty$). Plugging these four bounds,
    \[
    |E_{3,t}|
    \le k\int_0^1 \Big(2k\|S\|_\infty\Big)\Big(\tfrac{2}{d}\|S\|_\infty\Big)
    \Big(2\theta k\|S\|_\infty\Big)\Big(2\|S\|_\infty\Big)\,d\theta
    = \frac{8}{d}\,k^3\,\|S\|_\infty^4.
    \]
    Taking expectations, using \Cref{lem:gauss-max-final} with $q=4$, and summing over $t$,
    \begin{equation}\label{eq:E3t-sum}
    \sum_{t=1}^T \EE|E_{3,t}|
    \le \frac{8}{d}\,k^3\sum_{t=1}^T \EE\|S_{t-1}\|_\infty^4
    \le \frac{C}{d}\,k^3\sum_{t=1}^T (t-1)^2\log(e d)
    \le \frac{C}{d}\,k^3\,T^3\,\big(\log(e d)\big)^2.
    \end{equation}
    
    Combining \eqref{eq:Sigma1-upper-structure}, \eqref{eq:Sigma1-main}, \eqref{eq:E1t-sum}, \eqref{eq:E2-sum}, and \eqref{eq:E3t-sum}, we obtain
    \begin{equation}\label{eq:Sigma1-final}
    -2\Sigma_1
    \;\le\; \frac{k}{d} T^2  + C_1 k^2 T^{5/2}(\log(e d))^{3/2} + C_2 \frac{k^2}{d^{3/2}}T^{5/2}(\log(ed))^{3/2}\,e^{4k^2T}
    + C_3 \frac{k^3}{d} T^3(\log(e d))^2.
    \end{equation}
    If $k = \tilde{\mathcal{O}}(1/\sqrt{T d})$ for some absolute constant $c_\ast>0$, then
    \[
    -2\Sigma_1
    \;\le\; \tilde{\mathcal{O}}\left(\frac{T^{3/2}}{d^{3/2}}  + \frac{T^{3/2}}{d} + \frac{T^{3/2}}{d^{5/2}} +  \frac{T^{3/2}}{d^{7/2}}\right). \]
                            
Next, we provide an upper bound on $2\Sigma_2$.
    We split $\Sigma_2$ into two parts:
    \[
    \Sigma_2 = \Sigma_{2,0} + \Sigma_{2,\mathrm{pert}},
    \qquad
    \Sigma_{2,0} := \sum_{t=1}^T \EE\big[\langle J_0S_{t-1},e_{J^\star}\rangle\big],
    \quad
    \Sigma_{2,\mathrm{pert}} := \sum_{t=1}^T \EE\big[\langle (J_t(k)-J_0)S_{t-1},e_{J^\star}\rangle\big].
    \]
    
    {\color{red} \emph{Baseline negative term. ($\Sigma_{2,0}$)}}
    By \Cref{fact:conditional-expectation} and \Cref{lem:regression-extreme},
    \begin{equation}\label{eq:Sigma20-final-again}
    2\Sigma_{2,0}
    \;\le\; -\,c_0\,\frac{\sqrt{T}\,(T-1)\,\sqrt{\log d}}{d}.
    \end{equation}
        \begin{equation}\label{eq:Sigma20-final-again}
    2\Sigma_{2,0}
    \;=\; -\,\Theta\left(\,\frac{T^{3/2}\,\sqrt{\log d}}{d}\right).
    \end{equation}
    {\color{red} \emph{Perturbative term $\Sigma_{2,\mathrm{pert}}$.}}
    
    By \Cref{lem:null-direction},
        \[
        \big\langle (J_t(k)-J_0)S,\,e_{J^\star}\big\rangle
        \;=\; \big\langle R_t J_0S,\,e_{J^\star}\big\rangle,\qquad R_t:=J_t(k)-J_0.
        \]
        Hence, using Hölder in $(\infty,1)$ and $\|e_{J^\star}\|_1=1$,
        \[
        \big|\big\langle (J_t(k)-J_0)S,\,e_{J^\star}\big\rangle\big|
        \le \|R_tJ_0S\|_\infty
        \le \|R_t\|_{\infty\to\infty}\,\|J_0S\|_\infty.
        \]
        By Lemma~\ref{lem:softmax}\,(e) with $h=-kS$,
        \(
        \|R_t\|_{\infty\to\infty}
        =\|J_{\sm}(-kS)-J_0\|_{\infty\to\infty}
        \le 2k\|S\|_\infty.
        \)
        By the definition of $J_0$, $\|J_0S\|_\infty\le \frac{2}{d}\|S\|_\infty$. Therefore
        \[
        \big|\big\langle (J_t(k)-J_0)S,\,e_{J^\star}\big\rangle\big|
        \le \frac{4}{d}\,k\,\|S\|_\infty^2.
        \]
        Taking expectations and summing over $t$, and using Lemma~\ref{lem:gauss-max-final} with $q=2$,
        \[
        |\Sigma_{2,\mathrm{pert}}|
        \le \frac{4}{d}\,k\sum_{t=1}^T \EE\|S_{t-1}\|_\infty^2
        \le \frac{C}{d}\,k\sum_{t=1}^T (t-1)\log(e d)
        \le \frac{C}{d}\,k\,T^2\log(e d).
        \]

        Together, we have
        \[
        2\Sigma_2        \le -c_0\frac{\sqrt{T}\,(T-1)\,\sqrt{\log d}}{d}
        + \frac{C}{d}\,k\,T^2\log(e d). 
        \]
    If $k = \tilde{\mathcal{O}}(1/\sqrt{T d})$ for some absolute constant $c_\ast>0$, then
    \[
    2\Sigma_2
    \le -c_0\frac{\sqrt{T}\,(T-1)\,\sqrt{\log d}}{d}
    + \tilde{\mathcal{O}}\left(\frac{T^{3/2}}{d^{3/2}}\right). \]

    Therefore, $\mathcal L'(k) < 0$ for $k = \tilde{\mathcal{O}}(1/\sqrt{T d})$ as desired.

\end{proof}

\begin{theorem}[Positivity for $k = C/\sqrt{T} \mathrm{polylog}(d)$] \label{thm:positivity}
    There exists an absolute constant $C>0$ such that for all $T,d\ge 2$, if $k = \tilde{\Theta}(1/\sqrt{T})$, then $\mathcal L'(k)>0$. 
    \end{theorem}

\begin{proof}
First, we use \Cref{lem:derivative} and analyze $\Sigma_1$ and $\Sigma_2$ separately. First, we restate \Cref{eq:Sigma1-E123} as 
\begin{align}
    \langle J_t S,\pi-u\rangle
    &= \langle (J_0+R_t)S, \pi - u\rangle \notag\\
    &= \underbrace{\langle J_0S, \pi \rangle}_{N_{1,t}} 
    \;-\; \underbrace{\langle R_tS, \pi - u\rangle}_{N_{2,t}}.
    \label{eq:Sigma1-E123-again}
    \end{align}
First, we will bound $N_{2,t}$. 
\begin{claim}\label{prop:jac-policy-perturb}
    There exist absolute constants $c_0,C_0>0$ such that the following holds.
    For every $k$ with $0<k\le c_0/\sqrt{T}$,
    \begin{equation}\label{eq:jac-policy-main-bound}
    \Big|\EE\Big\langle\big(J_t(k)-J_t(0)\big)S_{t-1},\ \pi_t(k)-\pi_0\Big\rangle\Big|
    \ \le\ C_0\,T\,k\,d^{-3/2}\,\log(ed).
    \end{equation}
    \end{claim}
    
    \begin{proof}
    Again, we abbreviate $S:=S_{t-1}$, and write $\pi_t(k)=\sm(-kS)$ and $J_t(k)=J_{\sm}(-kS)$.
    We wish to bound the following term:
    \begin{equation}\label{eq:Qtk-def}
    Q_t(k)
    := \EE\big\langle (J_t(k)-J_t(0))S_{t-1},\ \pi_t(k)-\pi_0\big\rangle
    = \EE\langle R S,\ \pi_k-\pi_0\rangle.
    \end{equation}
    
    By Lemma~\ref{lem:null-direction}, for every $z\in\R^d$ we have
    \[
    \big(J_{\sm}(z)-J_0\big)\one_d = \pmb{0}_d.
    \]
    Then $S$ decomposes into a multiple of $\one_d$ plus $S_\perp\in\mathrm{range}(J_0)$, and Lemma~\ref{lem:null-direction} implies
    \begin{equation}\label{eq:RS-perp}
    R S = R S_\perp.
    \end{equation}
    
    Moreover, $\pi_k$ and $\pi_0$ are probability vectors, so $\|\pi_k-\pi_0\|_1\le 2$.
    Thus, by Hölder's inequality in $(\infty,1)$,
    \begin{equation}\label{eq:Qtk-holder}
    |\langle R S,\pi_k-\pi_0\rangle|
    = |\langle R S_\perp,\pi_k-\pi_0\rangle|
    \le \|R S_\perp\|_\infty\,\|\pi_k-\pi_0\|_1
    \le 2\,\|R S_\perp\|_\infty.
    \end{equation}
    Taking expectations in \eqref{eq:Qtk-holder} and using \eqref{eq:Qtk-def},
    \begin{equation}\label{eq:Qtk-reduction}
    |Q_t(k)|\ \le\ 2\,\EE\|R S_\perp\|_\infty.
    \end{equation}
    It remains to bound $\EE\|R S_\perp\|_\infty$.
    Along the ray $\theta\mapsto -\theta kS$ we have the second-order expansion
    \begin{equation}\label{eq:R-ray}
    R = J_{\sm}(-kS)-J_{\sm}(0)
    = DJ_{\sm}(0)[-kS] + \int_0^1 (1-\theta)\,D^2J_{\sm}(-\theta kS)[-kS,-kS]\,d\theta.
    \end{equation}
    Set
    \[
    L:=DJ_{\sm}(0)[-kS],\qquad
    \mathcal R:=\int_0^1 (1-\theta)\,D^2J_{\sm}(-\theta kS)[-kS,-kS]\,d\theta,
    \]
    so that $R=L+\mathcal R$.
    Then
    \begin{equation}\label{eq:RSperp-decomp}
    \|R S_\perp\|_\infty
    \le \|L S_\perp\|_\infty + \|\mathcal R S_\perp\|_\infty.
    \end{equation}
    We bound the expectations of the two terms in \eqref{eq:RSperp-decomp} separately.
    By Lemma~\ref{lem:DJ0},
    \[
    \|DJ(0)[h]\|_{\infty\to\infty} \le \frac{6}{d}\,\|h\|_\infty
    \qquad\forall h\in\R^d.
    \]
    With $h=-kS$ we obtain
    \[
    \|L\|_{\infty\to\infty}=\|DJ(0)[-kS]\|_{\infty\to\infty}
    \le \frac{6}{d}\,k\|S\|_\infty.
    \]
    By Lemma~\ref{lem:J0-infty},
    \[
    \|S_\perp\|_\infty
    = \|J_0 S\|_\infty
    \le \frac{2}{d}\,\|S\|_\infty.
    \]
    Hence
    \begin{equation}\label{eq:LSperp-pointwise}
    \|L S_\perp\|_\infty
    \le \|L\|_{\infty\to\infty}\,\|S_\perp\|_\infty
    \le \frac{12}{d^2}\,k\|S\|_\infty^2.
    \end{equation}
    
    Recall that $S=S_{t-1}\sim\cN(0,(t-1)I_{d\times d})$. By Lemma~\ref{lem:gauss-max-final} with $q=2$,
    \begin{equation}\label{eq:gauss-max-q2}
    \EE\|S\|_\infty^2
    \le C\,(t-1)\,\log(ed)
    \end{equation}
    for some absolute constant $C>0$.
    Taking expectations in \eqref{eq:LSperp-pointwise} and using \eqref{eq:gauss-max-q2} yields
    \begin{equation}\label{eq:LSperp-exp}
    \EE\|L S_\perp\|_\infty
    \ \le\ C_1\,\frac{k}{d^2}\,(t-1)\,\log(ed)
    \end{equation}
    for an absolute constant $C_1>0$.
    
    From Lemma~\ref{lem:softmax}\,(f), for each $\theta\in[0,1]$ and all $S$,
    \begin{equation}\label{eq:D2J-bound}
    \Big\|D^2J(-\theta kS)[-kS,-kS]\Big\|_{\infty\to\infty}
    \le \frac{16}{d}\,e^{2\theta k\|S\|_\infty}\,k^2\|S\|_\infty^2.
    \end{equation}
    Therefore,
    \begin{align}
    \|\mathcal R S_\perp\|_\infty
    &\le \int_0^1 (1-\theta)\,
    \Big\|D^2J(-\theta kS)[-kS,-kS]\Big\|_{\infty\to\infty}\,
    \|S_\perp\|_\infty\,d\theta \notag\\
    &\le \int_0^1 (1-\theta)\,\frac{16}{d}e^{2\theta k\|S\|_\infty}k^2\|S\|_\infty^2
    \cdot \frac{2}{d}\|S\|_\infty\,d\theta \notag\\
    &= \frac{32}{d^2}k^2\|S\|_\infty^3
    \int_0^1 (1-\theta)e^{2\theta k\|S\|_\infty}\,d\theta.
    \label{eq:Rsperp-int}
    \end{align}
    For $q>0$, a direct computation gives
    \[
    \int_0^1 (1-\theta)e^{2\theta kq}\,d\theta
    \le \frac{1}{2kq}\,e^{2kq}.
    \]
    Applying this in \eqref{eq:Rsperp-int} with $q=\|S\|_\infty$ yields
    \begin{equation}\label{eq:Rsperp-pointwise}
    \|\mathcal R S_\perp\|_\infty
    \le \frac{16}{d^2}k\,\|S\|_\infty^2 e^{2k\|S\|_\infty}.
    \end{equation}
    
    Now take expectations. By Cauchy--Schwarz,
    \begin{equation}\label{eq:RS-CS}
    \EE\big[\|S\|_\infty^2 e^{2k\|S\|_\infty}\big]
    \le \Big(\EE\|S\|_\infty^4\Big)^{1/2}
    \Big(\EE e^{4k\|S\|_\infty}\Big)^{1/2}.
    \end{equation}
    By Lemma~\ref{lem:gauss-max-final} with $q=4$,
    \begin{equation}\label{eq:gauss-max-q4}
    \EE\|S\|_\infty^4 \le C\,(t-1)^2 (\log(ed))^2.
    \end{equation}
    By the MGF bound for the sup (Lemma~\ref{lem:mgf-sup}),
    \begin{equation}\label{eq:mgf-sup}
    \EE e^{4k\|S\|_\infty}
    \le C'\,d\,\exp\big(C''\,k^2(t-1)\big)
    \end{equation}
    for absolute constants $C',C''>0$.
    Plugging \eqref{eq:gauss-max-q4} and \eqref{eq:mgf-sup} into \eqref{eq:RS-CS}, we obtain
    \begin{equation}\label{eq:RS-CS-final}
    \EE\big[\|S\|_\infty^2 e^{2k\|S\|_\infty}\big]
    \le C_2\,(t-1)\,\log(ed)\,\sqrt{d}\,\exp\big(C_3 k^2(t-1)\big)
    \end{equation}
    for some absolute constants $C_2,C_3>0$.
    Combining \eqref{eq:Rsperp-pointwise} and \eqref{eq:RS-CS-final} gives
    \begin{equation}\label{eq:Rsperp-exp}
    \EE\|\mathcal R S_\perp\|_\infty
    \le C_2\,\frac{k}{d^2}\,(t-1)\log(ed)\,\sqrt{d}\,\exp\big(C_3 k^2(t-1)\big)
    = C_2\,\frac{k}{d^{3/2}}\,(t-1)\log(ed)\,e^{C_3 k^2(t-1)}.
    \end{equation}
    
    From \eqref{eq:RSperp-decomp}, \eqref{eq:LSperp-exp}, and \eqref{eq:Rsperp-exp},
    \begin{equation}\label{eq:RSperp-exp-total}
    \EE\|R S_\perp\|_\infty
    \le C_1\,\frac{k}{d^2}\,(t-1)\log(ed)
    + C_2\,\frac{k}{d^{3/2}}\,(t-1)\log(ed)\,e^{C_3 k^2(t-1)}.
    \end{equation}
    Recalling \eqref{eq:Qtk-reduction}, we obtain
    \begin{equation}\label{eq:Qtk-prelim}
    |Q_t(k)|
    \le 2\,\EE\|R S_\perp\|_\infty
    \le C\,k\,(t-1)\log(ed)\Big(\frac{1}{d^2} + \frac{1}{d^{3/2}}e^{C_3 k^2(t-1)}\Big)
    \end{equation}
    for some absolute $C>0$.
    
    Now assume $0<k\le c_0/\sqrt{T}$ for a sufficiently small absolute constant $c_0>0$.
    Then, for all $t\le T$,
    \[
    k^2(t-1)\le k^2 T \le c_0^2,
    \qquad\Rightarrow\qquad
    e^{C_3 k^2(t-1)} \le e^{C_3 c_0^2}=:C_\ast,
    \]
    where $C_\ast$ is an absolute constant. Moreover $t-1\le T$, and, for all $d\ge 2$,
    \[
    \frac{1}{d^2}\le \frac{1}{d^{3/2}}.
    \]
    Thus, absorbing $C_\ast$ and the factor $2$ into a new constant $C_0$,
    \eqref{eq:Qtk-prelim} yields
    \[
    |Q_t(k)|
    \le C_0\,\frac{kT}{d^{3/2}}\,\log(ed),
    \]
    uniformly over all $t\in[T]$, all $d\ge 2$, and all $k\in(0,c_0/\sqrt{T}]$.
    This is exactly \eqref{eq:jac-policy-main-bound}, which completes the proof.
    \end{proof}
    
Now, we focus on calculating $N_{1,t}$. By \Cref{lem:softmax-expectation}, we have $N_{1,t} = \tilde{\Theta}(\frac{T^2k}{d})$. more specifically, $N_{1,t} \leq -C_2 \frac{T^2k}{d}$.
Therefore, we have 
\begin{align}
    -2\Sigma_1 \geq 2C_2 \frac{T^2k}{d} - C_0 \frac{kT^2}{d^{3/2}}\,\log(ed) \label{eq:Sigma1-final}
\end{align}

Now, we focus on calculating $\Sigma_2$. For $\Sigma_{2, 0}$, by \Cref{fact:conditional-expectation} and \Cref{lem:regression-extreme-lb}, we have $\Sigma_{2, 0} = \tilde{\Theta}(\frac{T^{1/2}(T-1)\sqrt{\log d}}{d})$. more specifically, $\Sigma_{2, 0} \geq -C_4 \frac{T^{1/2}(T-1)\sqrt{\log d}}{d}$ for some $C_4 > 0$.

\begin{proposition}\label{prop:sigma2pert-7/4}
    There exist absolute constants $C_1,C_2>0$ such that
    \begin{equation}\label{eq:sigma2-7/4}
    |\Sigma_{2,\mathrm{pert}}|
    \le \frac{C_1}{d^2}\,k\,T^2\log(ed)
    + \frac{C_3}{d^{3/2}}\,k\,T^{5/2}\log(ed)e^{4k^2 T}.
    \end{equation}
    In particular, with $k=1/\sqrt{T}$,
    \(
    |\Sigma_{2,\mathrm{pert}}|\le C\,d^{-7/4}\,T^{3/2}\,\mathrm{polylog}(d).
    \)
    \end{proposition}
    
    \begin{proof}
    Fix $t$ and write $S:=S_{t-1}$.
    By \Cref{lem:null-direction} and Hölder in $(\infty,1)$,
    \[
    |\langle R_t S,e_{J^\star}\rangle|
    \le \|R_t J_0S\|_\infty.
    \]
    where $R_t = J_t(k) - J_0$ defined in \Cref{prop:jac-expansion}
    Decompose $R_t$ into its linear and quadratic parts along the ray:
    \[
    R_t
    =\underbrace{D J(\pmb{0})[-kS]}_{=:L_t}
    +\underbrace{\int_0^1 (1-\theta)\,D^2 J(-\theta kS)[-kS,-kS]\,d\theta}_{=:\mathcal R_t}.
    \]
    Thus
    \[
    \|R_tJ_0S\|_\infty
    \;\le\; \|L_t J_0S\|_\infty \;+\; \|\mathcal R_t J_0S\|_\infty.
    \]
    
    \emph{(i) Linear piece.}
    By Lemma~\ref{lem:DJ0} and \Cref{lem:J0-infty},
    \[
    \|L_t J_0S\|_\infty
    \le \|D J(\pmb{0})[-kS]\|_{\infty\to\infty}\,\|J_0S\|_\infty
    \le \frac{6}{d}k\|S\|_\infty \cdot \frac{2}{d}\|S\|_\infty
    = \frac{12}{d^2}k\,\|S\|_\infty^2.
    \]
    Taking expectations and summing (using \Cref{lem:gauss-max-final} with $q=2$),
    \begin{equation}\label{eq:sig2-linear}
    \sum_{t=1}^T \EE \|L_t J_0S_{t-1}\|_\infty
    \;\le\; \frac{C_1}{d^2}\,k\,T^2\log(ed).
    \end{equation}
    
    \emph{(ii) Quadratic remainder, via Cauchy--Schwarz and MGF.}
    For a fixed $t$,
    \begin{align*}
    \|\mathcal R_t J_0S\|_\infty
    &\le \int_0^1 (1-\theta)\,\big\|D^2 J(-\theta kS)[-kS,-kS]\,J_0S\big\|_\infty\,d\theta
    \\
    &\leq \int_0^1 (1-\theta)\,\|D^2 J(-\theta kS)[-kS,-kS]\|_{\infty\to\infty}\,\|J_0S\|_\infty\,d\theta
    \\
    &\underset{(i)}{\leq} \int_0^1 (1-\theta)\,\frac{16}{d}\,e^{2\theta k\|S\|_\infty}k^2\|S\|_\infty^2\|J_0S\|_\infty d\theta
    \\
    &\underset{(ii)}{\leq} \int_0^1 (1-\theta)\,\frac{32}{d^2}\,e^{2\theta k\|S\|_\infty}k^2\|S\|_\infty^3 d\theta
    \\
    &\leq \int_0^1 \frac{32}{d^2}\,e^{2\theta k\|S\|_\infty}k^2\|S\|_\infty^3 d\theta =  \frac{32}{d^2}\,k^2\|S\|_\infty^3 \frac{1}{2k\|S\|_\infty}(e^{2k\|S\|_\infty} - 1) \leq \frac{16}{d^2}\,k\|S\|_\infty^2 e^{2k\|S\|_\infty}.
    \end{align*}
    where (i) is due to \Cref{lem:softmax}-(f) and (ii) is due to \Cref{lem:J0-infty}. Therefore, we have 
    \begin{align*}
    \sum_{t=1}^T \EE \|\mathcal R_t J_0S_{t-1}\|_\infty
    &\le \frac{16}{d^2}\,k\sum_{t=1}^T \EE[\|S_{t-1}\|_\infty^2 e^{2k\|S_{t-1}\|_\infty}]
    \\
    &\le \frac{16}{d^2}\,k\sum_{t=1}^T \sqrt{\EE[\|S_{t-1}\|_\infty^4]\EE[ e^{4k\|S_{t-1}\|_\infty}]}
    \\
    &\underset{(iii)}{\leq}\frac{C_1}{d^2}\,k\sum_{t=1}^T \sqrt{\EE[\|S_{t-1}\|_\infty^4]d \exp\left(8k^2(t-1)\right)} 
    \\
    &\underset{(iv)}{\leq} \frac{C_2}{d^2}\,k\sum_{t=1}^T \sqrt{(t-1)^2\log(ed)^2d \exp\left(8k^2(t-1)\right)} \leq \frac{C_3}{d^{3/2}}kT^2\log(ed) e^{4k^2T}
    \end{align*}
    where (iii) is due to \Cref{lem:mgf-sup} and (iv) is due to \Cref{lem:gauss-max-final} with $q=4$. Therefore, we have
    
    \begin{align*}
    |\Sigma_{2,\mathrm{pert}}|
    \le \frac{C_1}{d^2}\,k\,T^2\log(ed)
    + \frac{C_3}{d^{3/2}}\,k\,T^{5/2}\log(ed)e^{4k^2 T}.
    \end{align*}
    \end{proof} 
Therefore, we have 
\begin{align}
2 \Sigma_2 \geq -C_4 \frac{T^{1/2}(T-1)\sqrt{\log d}}{d}- \frac{C_5}{d^2}\,k\,T^2\log(ed) - \frac{C_6}{d^{3/2}}\,k\,T^{5/2}\log(ed)e^{4k^2 T} \label{eq:Sigma2-final}
\end{align}
and combining \eqref{eq:Sigma1-final} and \eqref{eq:Sigma2-final} with $k = \Theta(1/\sqrt{T} \mathrm{polylog}(d))$, we have $\mathcal L'(k) > 0$. 

\end{proof}
By \Cref{thm:smallk} and \Cref{thm:positivity}, we have $\mathcal L'(k) < 0$ for $k = \tilde{\mathcal{O}}(1/\sqrt{T d})$ and $\mathcal L'(k) > 0$ for $k = \Theta(1/\sqrt{T} \mathrm{polylog}(d))$. Therefore, there exists $k^\star$ in the interval $(\tilde{{\Theta}}(1/\sqrt{T d}), \tilde{\Theta}(1/\sqrt{T}))$ such that $\mathcal L'(k^\star) = 0$. 
\end{proof}

\section{Deferred Lemmas}
We provide the proofs of lemmas related to Gaussian distributions and the $\softmax$ function that are used in the main text for the sake of completeness.

\begin{lemma}[Permutation commutant]\label{lem:permutation-commutant}
    Let $A\in\mathbb R^{d\times d}$ satisfy
    \[
    PAP^\intercal \;=\; A
    \qquad\text{for every permutation matrix }P.
    \]
    Then there exist scalars $\alpha,\beta\in\mathbb R$ such that
    \[
    A=\alpha I_{d\times d}+\beta \pmb{1}_d\pmb{1}_d^\intercal.
    \]
    If, in addition, $\pmb{1}_d^\intercal A=\pmb{0}_d^\intercal$ or
    $A\pmb{1}_d=\pmb{0}_d$, then
    \[
    A=\alpha\left(I_{d\times d}
      -\frac{1}{d}\pmb{1}_d\pmb{1}_d^\intercal\right).
    \]
    Consequently,
    \[
    \left\langle A,-I_{d\times d}\right\rangle_F=-\alpha(d-1)
    \]
    in this zero-sum case.
    \end{lemma}
    
    \begin{proof}
    Fix $A$ satisfying $PAP^\intercal=A$ for all permutation matrices $P$.
    For any two coordinates $i,j$, a transposition exchanging them shows
    $A_{ii}=A_{jj}$, so all diagonal entries are equal.  For any two ordered
    pairs $(i,j)$ and $(r,s)$ with $i\neq j$ and $r\neq s$, there is a
    permutation sending $i$ to $r$ and $j$ to $s$, so $A_{ij}=A_{rs}$.  Thus
    all off-diagonal entries are equal, which is equivalent to
    $A=\alpha I_{d\times d}+\beta\pmb{1}_d\pmb{1}_d^\intercal$ after
    reparameterizing the common diagonal and off-diagonal values.

    If $\pmb{1}_d^\intercal A=\pmb{0}_d^\intercal$, then
    $\pmb{1}_d^\intercal A=(\alpha+d\beta)\pmb{1}_d^\intercal$, so
    $\beta=-\alpha/d$.  The same conclusion follows from
    $A\pmb{1}_d=\pmb{0}_d$.  Finally,
    $\langle I_{d\times d}-d^{-1}\pmb{1}_d\pmb{1}_d^\intercal,
    I_{d\times d}\rangle_F=d-1$, giving the displayed Frobenius identity.
    \end{proof}

\begin{lemma}[$\softmax$ Jacobian: spectral, $\infty$, and $1$ bounds]\label{lem:softmax}
For $z\in\mathbb R^d$, let 
\[
p=\softmax(z), 
\qquad 
J_{\mathrm{sm}}(z)=\Diag(p)-pp^\intercal.
\]
Then:
\begin{enumerate}
    \item[\textup{(a)}] $J_{\mathrm{sm}}(z)$ is symmetric positive semidefinite and $J_{\mathrm{sm}}(z)\pmb{1}_d=0$.

    \item[\textup{(b)}] Operator, $\infty\to\infty$, and $1\to 1$ bounds:
    \[
    \|J_{\mathrm{sm}}(z)\|_{\op}
    = \|J_{\mathrm{sm}}(z)\|_{\infty\to\infty}
    = \|J_{\mathrm{sm}}(z)\|_{1\to 1}
    \le \frac12.
    \]

    \item[\textup{(c)}] $\ell_2$–Lipschitz continuity of $\softmax$:
    \[
    \|\softmax(z)-\softmax(z')\|_2 \;\le\; \frac12\|z-z'\|_2.
    \]

    \item[\textup{(d)}] $\ell_2$–Lipschitz continuity of the Jacobian:
    \[
        \|J_{\mathrm{sm}}(z)-J_{\mathrm{sm}}(z')\|_{\op}
        \;\le\; \frac32\,\|z-z'\|_2.
    \]

    \item[\textup{(e)}] Derivative bounds in induced norms:
    \[
    \|D J_{\mathrm{sm}}(z)[h]\|_{\infty\to\infty} \le 2\,\|h\|_\infty,
    \qquad
    \|D J_{\mathrm{sm}}(z)[h]\|_{1\to 1} \le 2 \|h\|_\infty,
    \]
    and therefore $\norm{J_{\sm}(\theta z) - J_{\sm}(\pmb{0}_d)}_{\infty \to \infty} \le 2 \theta \norm{z}_\infty$.

    \item[\textup{(f)}] $\infty \to \infty$ bound of the second derivative of $J_{\sm}$:
        \begin{equation}\label{eq:D2J-inf}
            \|D^2 J_{\sm}(z)[h,h]\|_{\infty\to\infty}
            \;\le\; \frac{16}{d}\,e^{2\|z\|_\infty}\|h\|_\infty^2
            \end{equation}
            for some absolute constant $C>0$.
\end{enumerate}
\end{lemma}

\begin{proof}
\noindent{\color{blue}\textbf{{(a) PSD and nullspace property.}}}
For any $x\in\mathbb R^d$,
\[
x^\intercal J_{\mathrm{sm}}(z)x
= \sum_i p_i x_i^2 - \Bigl(\sum_i p_i x_i\Bigr)^2
= \mathrm{Var}_p(x)\ge 0.
\]
Also,
\[
J_{\mathrm{sm}}(z)\pmb{1}_d
= \Diag(p)\pmb{1}_d - p(p^\intercal \pmb{1}_d)
= p - p(1)=0.
\]

\noindent{\color{blue}\textbf{{(b) Spectral-norm bound.}}}
Since $J_{\sm}(z)$ is symmetric,
\[
\|J_{\sm}(z)\|_{\op}
= \sup_{\|x\|_2=1}\mathrm{Var}_p(x).
\]
For any $x$ with $\|x\|_2=1$ and $m=\min x_i$, $M=\max x_i$,
\[
\mathrm{Var}_p(x)\le \frac{(M-m)^2}{4}.
\]
But
\[
(M-m)^2
= (x_i-x_j)^2
\le 2(x_i^2+x_j^2)\le 2\|x\|_2^2=2.
\]
Thus
\[
\mathrm{Var}_p(x) \le \frac12.
\]
Moreover, for each row $i$,
\[
\sum_j |(J_{\sm}(z))_{ij}|
= p_i(1-p_i)+p_i\sum_{j\neq i}p_j
=2p_i(1-p_i)\le \frac12,
\]
since $x(1-x)\le 1/4$; the same bound holds for each column by symmetry. Thus the $\infty\to\infty$ and $1\to 1$ operator norms also equal $\|J_{\sm}(z)\|_{\op}$.

\noindent{\color{blue}\textbf{{(c) $\softmax$ is $1/2$–Lipschitz in $\ell_2$.}}}
For $\gamma(\theta)=(1-\theta)z+\theta z'$,
\[
\softmax(z')-\softmax(z)
=\int_0^1 J_{\sm}(\gamma(\theta))(z'-z)\,d\theta.
\]
Take norms and use (b):
\[
\|\softmax(z')-\softmax(z)\|_2
\le \frac12\|z'-z\|_2.
\]

\noindent{\color{blue}\textbf{{(d) Jacobian Lipschitz bound.}}}
Differentiate 
\[
J_{\sm}(z)=\Diag(p(z)) - p(z)p(z)^\intercal.
\]
For $q:=Dp(z)[h]=J_{\sm}(z)h$,
\[
DJ_{\sm}(z)[h]
= \Diag(q)-(q p^\intercal + p q^\intercal).
\]
Operator norms satisfy
\[
\|\Diag(q)\|_{\op}\le \|q\|_\infty\le \|q\|_2,
\qquad
\|q p^\intercal\|_{\op}=\|q\|_2\|p\|_2\le \|q\|_2,
\]
and likewise for $pq^\intercal$. Hence
\[
\|DJ_{\sm}(z)[h]\|_{\op} \le 3\|q\|_2.
\]
Using $\|q\|_2=\|J_{\sm}(z)h\|_2\le \frac12\|h\|_2$ from (b),
\[
\|DJ_{\sm}(z)[h]\|_{\op}\le \frac32\|h\|_2.
\]
Integrate along the segment from $z$ to $z'$ to conclude
\[
\|J_{\sm}(z)-J_{\sm}(z')\|_{\op}\le \frac32\|z-z'\|_2.
\]

\noindent {\color{blue}\textbf{{(e) Bounds for $DJ_{\sm}(z)[h]$ in $\infty$– and $1$–norms.}}}
From above,
\begin{equation}
DJ_{\sm}(z)[h]
= \Diag(q)-(q p^\intercal + p q^\intercal). \label{eq:DJsm-h}
\end{equation}

\textbf{Step 1: generic bound.}
Using induced–norm identities,
\[
\|DJ_{\sm}(z)[h]\|_{\alpha\to\alpha}
\le 2\|q\|_\infty + \|q\|_1,
\qquad \alpha\in\{1,\infty\}.
\]

\textbf{Step 2: bounds on $q$.}
From (b):
\[
\|q\|_\infty=\|J_{\sm}(z)h\|_\infty
\le \frac12\|h\|_\infty
\]

We also need the dimension-free inequality
\[
\|q\|_1=\sum_i p_i|h_i-\langle p,h\rangle|
\le \|h\|_\infty.
\tag{$\star$}
\]
This follows by thinking $p$ as a probability distribution and define $X$ as a discrete random variable with taking $h_i$ with probability $p_i$, then $\EE[|X - \EE[X]|] \leq \sqrt{\EE[|X - \EE[X]|^2]} = \sqrt{\text{Var}(X)} \leq \norm{h}_\infty$.

\textbf{Step 3: final estimates.}

For the $\infty\to\infty$ norm:
\[
\|DJ_{\sm}(z)[h]\|_{\infty\to\infty}
\le 2\cdot \frac12\|h\|_\infty + \|h\|_\infty
= 2\|h\|_\infty.
\]

For the $1\to 1$ norm:
\[
\|DJ_{\sm}(z)[h]\|_{1\to1} \leq 2 \|h\|_\infty.
\]

This completes the proof. 

\noindent{\color{blue}\textbf{{(f) $\infty \to \infty$ bound of the second derivative of $J_{\sm}$.}}}
First, we differentiate \eqref{eq:DJsm-h} in direction $h$.
Recall that for $p(z)=\softmax(z)$,
\[
J_{\sm}(z) \;=\; \Diag(p) - p p^\intercal,
\]
and for a fixed direction $h$ we write
\[
q := Dp(z)[h] = J_{\sm}(z) h, \qquad a := D^2 p(z)[h,h].
\]
From \eqref{eq:DJsm-h} we have
\[
DJ_{\sm}(z)[h]
=
\Diag(q) - qp^\intercal - p q^\intercal.
\]
Differentiating once more in the direction $h$ gives
\begin{align*}
D^2J_{\sm}(z)[h,h]
&= D(\Diag(q))[h]- D(qp^\intercal)[h]- D(pq^\intercal)[h]
\\
&= \Diag(Dq[h]) - \bigl((Dq[h])p^\intercal +  q(Dp[h])^\intercal\bigr)
      - \bigl(p(Dq[h])^\intercal +  (Dp[h])q^\intercal\bigr).
\end{align*}
Since $Dp[h]=q$ and $Dq[h]=D^2p(z)[h,h]=a$, this becomes
\begin{align*}
D^2J_{\sm}(z)[h,h]
&= \Diag(a) - \bigl(a p^\intercal +  qq^\intercal\bigr)
      - \bigl(p a^\intercal +  qq^\intercal\bigr)
\\
&= \Diag(a) - a p^\intercal - p a^\intercal - 2 q q^\intercal.
\end{align*}
Thus
\[
D^2 J_{\sm}(z)[h,h]
=
\Diag(a) - a p^\intercal - p a^\intercal - 2 q q^\intercal,
\quad
p:=p(z),\quad q:=J_{\sm}(z)h,\quad a:=D^2 p(z)[h,h].
\]
By standard calculation, we can derive that 
$a_i =   p_i(z)((h_i - \mu)^2 - \sigma^2)$, where $\mu = \sum_{j=1}^d p_j(z) h_j$ and $\sigma^2 = \sum_{j=1}^d p_j(z)(h_j - \mu)^2$. Due to the definition of $\norm{h}_\infty$, we have $|h_i - \mu| \leq 2\norm{h}_\infty$ and $|\sigma^2| \leq 4\norm{h}_\infty^2$. Thus 
\[
|a_i| \leq p_i(z)(4\norm{h}_\infty^2 + 4\norm{h}_\infty^2) = 8p_i(z)\norm{h}_\infty^2.
\]
Therefore,
\[
\|\Diag(a)\|_{\infty\to\infty} \leq 8\|p\|_\infty\norm{h}_\infty^2.
\]
Similarly, we have $\|a p^\intercal\|_{\infty\to\infty} \leq 8\|p\|_\infty\norm{h}_\infty^2$ and $\|p a^\intercal\|_{\infty\to\infty} \leq 8\|p\|_\infty\norm{h}_\infty^2$.

Now, we bound $\|q q^\intercal\|_{\infty\to\infty}$ by bounding $\|q\|_\infty$ and $\|q\|_1$, since $\|q q^\intercal\|_{\infty\to\infty}\le \|q\|_\infty\|q\|_1$.
We bound $\|q\|_\infty$ by using the definition of $q$:
\[
q = J_{\sm}(z)h = \Diag(p(z))h - p(z)(p(z)^\intercal h)
\]
Therefore we have the followings:
\begin{align*}
|q_i| &\leq p_i(z)|h_i| + p_i(z)\sum_{j=1}^d p_j(z)|h_j| \leq p_i(z) \norm{h}_\infty + p_i(z) \norm{h}_\infty \leq 2\norm{p(z)}_\infty \norm{h}_\infty \\
\norm{q}_1 &\leq \norm{\Diag(p(z))h}_1 + \norm{p(z)(p(z)^\intercal h)}_1 \leq \norm{h}_\infty + \norm{h}_\infty \leq 2\norm{h}_\infty 
\end{align*}
Thus, we have $\|q\|_\infty \leq 2\norm{p(z)}_\infty \norm{h}_\infty$ and $\|q\|_1 \leq 2\norm{h}_\infty$, therefore $\|q q^\intercal\|_{\infty\to\infty} \leq 4\norm{p(z)}_\infty \norm{h}_\infty^2$.

Collecting all the bounds, we have
\[
\|D^2 J_{\sm}(z)[h,h]\|_{\infty\to\infty} \leq 32\|p(z)\|_\infty\norm{h}_\infty^2.
\]

Now, we have 
\begin{align*}
    p_i(z) = p_i (z - \bar{z}) = \frac{e^{z_i - \bar{z}}}{\sum_{j=1}^d e^{z_j - \bar{z}}}
\end{align*}
where $\bar{z} = \frac1d\sum_{j=1}^d z_j$. We know that $\frac1d\sum_{j=1}^d e^{z_j - \bar{z}} \geq 1$ due to Jensen's inequality, therefore $\|p(z)\|_\infty \leq \frac{1}{d}e^{2\|z\|_\infty}$, which concludes the proof.

\end{proof}

\begin{lemma}\label{lem:J0-refined-infty}
    Define $J_0 = \Diag(u)-uu^\intercal$ with $u=\frac1d\pmb{1}_d$. For every $x\in\RR^d$,
    \[
    \|J_0 x\|_\infty \;\le\; \frac{1}{d}\Big(\|x\|_\infty + |\pmb{1}_d^\intercal x|\Big)
    \]
    \end{lemma}
    
    \begin{proof}
    Since $J_0=\frac1d I_{d\times d}-\frac1{d^2}\pmb{1}_d\pmb{1}_d^\intercal$, we have 
    \(
    (J_0x)_i=\frac1d\big(x_i-(\pmb{1}_d^\intercal x)/d\big).
    \)
    Hence
    \(
    \|J_0x\|_\infty=\frac1d\max_i |x_i-\bar x|
    \le \frac1d\big(\|x\|_\infty+|\bar x|\big).
    \)
    \end{proof}
\begin{lemma}[Null-direction reduction]\label{lem:null-direction}
    For every $z\in\RR^d$ and every $x\in\RR^d$,
    \[
    \big(J_{\sm}(z)-J_0\big)\,x \;=\; \big(J_{\sm}(z)-J_0\big)\,J_0x.
    \]
    Equivalently, with $R_t:=J_t(k)-J_0$ and $A_\theta:=J_{\sm}(\theta z)-J_0$ (for any $z$),
    \[
    R_t x \;=\; R_t J_0 x,
    \qquad
    A_\theta x \;=\; A_\theta J_0 x.
    \]
    \end{lemma}
    
    \begin{proof}
    By \Cref{lem:softmax}-(a), $J_{\sm}(z)\pmb{1}_d=0$ for all $z$, and also $J_0\pmb{1}_d=0$. Hence
    \(
    \big(J_{\sm}(z)-J_0\big)\pmb{1}_d=0.
    \)
    Write $x=J_0x+\alpha\pmb{1}_d$, where $\alpha:=\tfrac1d\pmb{1}_d^\intercal x$. Then
    \[
    \big(J_{\sm}(z)-J_0\big)\,x
    = \big(J_{\sm}(z)-J_0\big)\,J_0x + \alpha\,\big(J_{\sm}(z)-J_0\big)\pmb{1}_d
    = \big(J_{\sm}(z)-J_0\big)\,J_0x.\qedhere
    \]
    \end{proof}

\begin{lemma}\label{lem:globalsign}
Let $x\in\mathbb R^d$ and $\alpha\ge 0$. Set $\pi:=\operatorname{\softmax}(-\alpha x)$ and $p:=\diag(\pi)-\pi\pi^\intercal$. Then $\langle p x,\pi\rangle \le 0$.
\end{lemma}

\begin{proof}
Sort indices so that $x_1\le x_2\le\cdots\le x_d$. Then $\pi_1\ge \pi_2\ge\cdots\ge \pi_d$ because $\pi_i\propto e^{-\alpha x_i}$.
Write $\mu_\pi:=\sum_i \pi_i x_i$ and note
\[
\langle p x,\pi\rangle
= \sum_{i=1}^d \pi_i^2 (x_i-\mu_\pi)
= \Big(\sum_{i=1}^d \pi_i^2\Big)\left(\underbrace{\frac{\sum_i \pi_i^2 x_i}{\sum_i \pi_i^2}}_{=:~\mu_{\pi^2}} - \mu_\pi\right).
\]
Let $q:=\pi^2/\sum_j \pi_j^2$ be $\pi$ squared and renormalized; then $q$ is \emph{more concentrated} than $\pi$, i.e. $\pi\succ q$ in the majorization order (Hardy–Littlewood–Pólya).
Since $(x_i)$ is increasing, majorization implies
\[
\mu_\pi=\sum_i \pi_i x_i \ \ge\ \sum_i q_i x_i \ =\ \mu_{\pi^2}
\quad\Longrightarrow\quad \mu_{\pi^2} - \mu_\pi \ \le\ 0,
\]
hence $\langle p x,\pi\rangle\le 0$.
\end{proof}

\begin{lemma}\label{lem:J0-infty}
    Let $J_0 = \Diag(u)-uu^\intercal$ with $u=\frac1d\pmb{1}_d$. Then
    \[
    \|J_0\|_{\infty\to\infty} = \|J_0\|_{1\to 1} \;=\; \frac{2(d-1)}{d^2}\;\le\;\frac{2}{d}, \qquad \norm{J_0}_{\op} = \frac{1}{d}
    \]
    \end{lemma}
    
    \begin{proof}
    The $i$th row of $J_0$ has diagonal entry $(d-1)/d^2$ and off-diagonal entries $-1/d^2$.
    Thus row $i$ has $\ell_1$-norm
    \[
    \frac{d-1}{d^2} + (d-1)\frac{1}{d^2} = \frac{2(d-1)}{d^2}.
    \]
    The same holds for column sums, yielding the claim. Lastly, the operator norm is the largest singular value, which is the largest eigenvalue of $J_0 J_0^\intercal = \frac{1}{d^2} \pmb{1}_d \pmb{1}_d^\intercal$, so $\norm{J_0}_{\op} = \frac{1}{d}$.
    \end{proof}
    \begin{lemma}\label{lem:DJ0}
        For all $h$, $\|D J_{\sm}(\pmb{0})[h]\|_{\infty\to\infty}\le \frac{6}{d}\|h\|_\infty$.
        \end{lemma}
        
        \begin{proof}
        For all $h$,
        \[
        D J_{\sm}(\pmb{0})[h]=\Diag(J_0h)-(J_0h)u^\intercal-u(J_0h)^\intercal,\quad u=\tfrac1d\pmb{1}_d,
        \]
        Using \Cref{lem:J0-infty}, we have $\|J_0h\|_\infty\le \frac{2}{d}\|h\|_\infty$, so
        \[
        \|D J_{\sm}(\pmb{0})[h]\|_{\infty\to\infty}\le \frac{6}{d}\|h\|_\infty.
        \]
        \end{proof}
    \begin{lemma}[Gaussian supremum moments]\label{lem:gauss-max-final}
        Let $Z\sim\cN(\pmb{0}_d, I_{d\times d})$ and $q\ge 1$. There exists $C_q>0$, independent of $d$, such that
        \[
        \EE\|Z\|_\infty^q \;\le\; C_q\,(\log(e d))^{q/2}.
        \]
        \end{lemma}
        
        \begin{proof}
        By a union bound,
        \(
        \PP(\|Z\|_\infty\ge t)\le 2d\,e^{-t^2/2}
        \).
        Integrating by parts gives $\EE\|Z\|_\infty^q \le C_q(\log(e d))^{q/2}$ for some $C_q$.
        \end{proof}
        
\begin{lemma}\label{lem:DJ0}
    For all $h$, $\|D J_{\sm}(\pmb{0})[h]\|_{\infty\to\infty}\le \frac{6}{d}\|h\|_\infty$.
    \end{lemma}
    
    \begin{proof}
    For all $h$,
    \[
    D J_{\sm}(\pmb{0})[h]=\Diag(J_0h)-(J_0h)u^\intercal-u(J_0h)^\intercal,\quad u=\tfrac1d\pmb{1}_d,
    \]
    Using \Cref{lem:J0-infty}, we have $\|J_0h\|_\infty\le \frac{2}{d}\|h\|_\infty$, so
    \[
    \|D J_{\sm}(\pmb{0})[h]\|_{\infty\to\infty}\le \frac{6}{d}\|h\|_\infty.
    \]
    \end{proof}
        
    \begin{lemma}\label{lem:regression-extreme}
        There exists an absolute constant $c_{\mathrm{ev}}\in(0,1)$ such that
        \[
        \EE\!\left[\min_{1\le j\le d} G_j\right]
        \;\le\; -\,c_{\mathrm{ev}}\sqrt{\log d},
        \qquad d\ge 2,
        \]
        for i.i.d.\ $G_j\sim\mathcal N(0,1)$.
        \end{lemma}
        
        \begin{proof}

We define $M_d := \max_{1\le j\le d} G_j$ and $X_d := \min_{1\le j\le d} G_j$. Since the $G_j$ are i.i.d.\ standard normals, $\PP(M_d \le a) = \Phi(a)^d$, where $\Phi$ denotes the standard normal cdf.
        
        We use the classical Mills ratio inequality \citep{gordon1941values}: for all $x>0$,
        \begin{equation}\label{eq:mills}
        1-\Phi(x) \;\ge\; \frac{x}{1+x^2}\,\varphi(x),
        \end{equation}
        where $\varphi(x)=\frac{1}{\sqrt{2\pi}}e^{-x^2/2}$ is the standard normal density.
        Set $a:=\sqrt{\log d}$.  Then
        \[
        \varphi(a)=\frac{1}{\sqrt{2\pi}}d^{-1/2},
        \qquad
        \frac{a}{1+a^2}
        =\frac{\sqrt{\log d}}{\,1+\log d\,}
        \;\ge\; \frac{c_0}{\sqrt{\log d}}
        \]
        for some universal constant $c_0>0$ and all $d\ge 2$. 
        Using \eqref{eq:mills},
        \[
        1-\Phi(a)
        \;\ge\; c_0\,\frac{d^{-1/2}}{\sqrt{\log d}}.
        \]
        Thus
        \[
        \PP(M_d \le a)
        = \Phi(a)^d
        = \bigl(1-(1-\Phi(a))\bigr)^d
        \le \left(1- c_0\,\frac{d^{-1/2}}{\sqrt{\log d}}\right)^d
        \le \exp\!\left(-c_0\,\frac{d^{1/2}}{\sqrt{\log d}}\right).
        \]
        Therefore,
        \[
        \PP\!\left(X_d \le -a\right)
        = 1 - \PP(M_d \le a)
        \ge 1 - \exp\!\left(-c_0\,\frac{d^{1/2}}{\sqrt{\log d}}\right).
        \]
        For all sufficiently large $d$, the right-hand side is at least $1/2$, so for such $d$,
        \[
        \EE[X_d]
        \;\le\; -a\,\PP(X_d \le -a)
        \;\le\; -\frac12 \sqrt{\log d}.
        \]
        Since $\EE[X_d]\le 0$ for each fixed $d$ and the ratio
        $-\EE[X_d]/\sqrt{\log d}$ is bounded away from $0$ on the finite range
        $2\le d\le d_0$, we obtain a universal constant 
        $c_{\mathrm{ev}}\in(0,1)$ such that
        \[
        \EE[X_d]\le -c_{\mathrm{ev}}\sqrt{\log d}
        \qquad\text{for all } d\ge 2.
        \]
        \end{proof}
        \begin{lemma}\label{lem:regression-extreme-lb}
            There exists an absolute constant $C_{\mathrm{ev}}>0$ such that
            \[
            \EE\!\left[\min_{1\le j\le d} G_j\right]
            \;\ge\; -\,C_{\mathrm{ev}}\sqrt{\log d},
            \qquad d\ge 2,
            \]
            for i.i.d.\ $G_j\sim\mathcal N(0,1)$.
            \end{lemma}
            
            \begin{proof}
            It suffices to upper bound $\EE[M_d]$ by $C\sqrt{\log d}$ for some absolute constant $C>0$.
            
            For any $t>0$,
            \[
            \PP(M_d\ge t)
            = \PP\Bigl(\bigcup_{j=1}^d\{G_j\ge t\}\Bigr)
            \;\le\; \sum_{j=1}^d \PP(G_j\ge t)
            = d\,\PP(G_1\ge t) \leq d e^{-t^2/2}.
            \]

            Next, decompose $M_d$ into its positive and negative parts:
            \[
            M_d^+ := \max\{M_d,0\},\qquad
            M_d^- := \max\{-M_d,0\},
            \]
            so that $M_d = M_d^+ - M_d^-$ and hence
            \[
            \EE[M_d] = \EE[M_d^+] - \EE[M_d^-] \;\le\; \EE[M_d^+].
            \]
            Since $M_d^+\ge 0$, we can use the standard tail integral representation for nonnegative random variables:
            \[
            \EE[M_d^+] = \int_0^\infty \PP(M_d^+ \ge t)\,dt
                       = \int_0^\infty \PP(M_d \ge t)\,dt.
            \]
            Therefore,
            \[
            \EE[M_d]
            \;\le\; \EE[M_d^+]
            = \int_0^\infty \PP(M_d \ge t)\,dt.
            \]
            
            Fix $x:=\sqrt{2\log d}$ and split the integral:
            \[
            \EE[M_d]
            \le \int_0^{x} 1\,dt
            \;+\; \int_{x}^\infty \PP(M_d\ge t)\,dt
            \;\le\; x \;+\; d\int_{x}^\infty e^{-t^2/2}\,dt.
            \]

            We bound the Gaussian tail using
            \[
            \int_{x}^\infty e^{-t^2/2}\,dt
            \;\le\; \frac{1}{x}e^{-x^2/2},
            \qquad x>0,
            \]
            which follows from integrating by parts or from a standard inequality for Gaussian tails. Plugging $x=\sqrt{2\log d}$, we get
            \[
            d\int_{x}^\infty e^{-t^2/2}\,dt
            \;\le\;
            d\cdot \frac{1}{x}e^{-x^2/2}
            =
            d\cdot \frac{1}{\sqrt{2\log d}}\cdot e^{-\log d}
            =
            \frac{1}{\sqrt{2\log d}}
            \;\le\; 1,
            \]
            for all $d\ge 2$. Therefore,
            \[
            \EE[M_d]
            \;\le\; x + 1
            \;=\; \sqrt{2\log d} + 1
            \;\le\; (\sqrt{2}+1)\sqrt{\log d}
            \;\le\; 3\sqrt{\log d},
            \qquad d\ge 2.
            \]
            \end{proof}

        \begin{lemma}[MGF for Gaussian suprema]\label{lem:mgf-sup}
            Let $G\sim\cN(\pmb{0}_d,tI_{d\times d})$. For any $a\ge 0$,
            \[
            \EE \exp\!\big(a\|G\|_\infty\big) \;\le\; 2d\,\exp\!\Big(\frac{a^2t}{2}\Big).
            \]
            \end{lemma}
            \begin{proof}
            By the union bound,
            $\exp(a\|G\|_\infty)=\max_i \exp(a|G_i|)\le \sum_i \exp(a|G_i|)$,
            so $\EE e^{a\|G\|_\infty}\le d\,\EE e^{a|G_1|}$. For $G_i\sim\cN(0,t)$,
            $\EE e^{a|G_i|}=2\int_0^\infty e^{ax}\frac{1}{\sqrt{2\pi t}}e^{-x^2/(2t)}dx
            \le 2e^{a^2t/2}$.
            \end{proof}
        \begin{lemma}[Stein's lemma for Gaussian vectors]\label{lem:stein} \citep{stein1981estimation}
            Let $X\sim\mathcal N(\pmb{0}_d,I_{d\times d})$ and let $f:\mathbb R^d\to\mathbb R$ be differentiable with 
            $\mathbb E\| \nabla f(X)\|_2 < \infty$ and $\mathbb E|X_i f(X)|<\infty$ for all $i$.
            Then for every $v\in\mathbb R^d$,
            \[
            \mathbb E\!\left[\,\langle X,v\rangle\, f(X)\,\right]
            \;=\;
            \mathbb E\!\left[\,\langle v,\nabla f(X)\rangle\,\right].
            \]
            Equivalently, coordinatewise:
            \[
            \mathbb E[\,X_i f(X)\,] = \mathbb E[\,\partial_i f(X)\,],
            \qquad i=1,\dots,d.
            \]
            \end{lemma}
            
            \begin{proof}
            We give the proof in one dimension; the multivariate case follows by linearity.
            
            Let $Z\sim\mathcal N(0,1)$ with density $\varphi(z)=\frac{1}{\sqrt{2\pi}}e^{-z^2/2}$.
            Assume $\mathbb E|Z f(Z)|<\infty$ and $\mathbb E|f'(Z)|<\infty$.
            Then
            \[
            \mathbb E[Z f(Z)]
            = \int_{-\infty}^{\infty} z f(z)\,\varphi(z)\,dz.
            \]
            Since $\varphi'(z) = -z\varphi(z)$, we may write
            \[
            z\varphi(z) = -\varphi'(z),
            \]
            so that
            \[
            \mathbb E[Z f(Z)]
            = -\int f(z)\,\varphi'(z)\,dz.
            \]
            Integrating by parts,
            \[
            -\int f(z)\,\varphi'(z)\,dz
            = -\left[f(z)\varphi(z)\right]_{-\infty}^{\infty}
              +\int f'(z)\varphi(z)\,dz.
            \]
            Because $f(z)\varphi(z)\to 0$ as $|z|\to\infty$ (Gaussian density dominates any polynomially growing $f$), the boundary term vanishes, and we obtain
            \[
            \mathbb E[Z f(Z)]
            = \int f'(z)\varphi(z)\,dz
            = \mathbb E[f'(Z)].
            \]
            
            For the $d$-dimensional case, apply the one-dimensional identity to each coordinate, conditioning on the remaining coordinates. Linearity of expectation and $\langle X,v\rangle=\sum_i v_i X_i$ then give
            \[
            \mathbb E[\langle X,v\rangle f(X)]
            = \sum_{i=1}^d v_i\,\mathbb E[X_i f(X)]
            = \sum_{i=1}^d v_i\,\mathbb E[\partial_i f(X)]
            = \mathbb E[\langle v,\nabla f(X)\rangle].
            \]
            \end{proof}

\begin{lemma}\label{lem:softmax-expectation}
    Let $a>0, d\geq 2$ and $X \sim \mathcal N(0,aI_{d\times d})$, and $\pi(k) = \softmax(-kX)$. Define $F_d(a,k)\ \coloneqq\ \mathbb{E}\,\big\langle J_0 X,\,\pi(k)\big\rangle.$
If $k = \tilde{\cO}\left(1/\sqrt{a}\right)$, $F_d(a,k) \;=\; -\,\Theta\!\big(a k/d\big).$
\end{lemma}
    
    \begin{proof}
        First, by the definition of $J_0$ and \Cref{lem:stein}, 
        \begin{align*}
        F_d(a,k) &= \frac{1}{d} \EE \langle X, \pi(k) \rangle = \frac{1}{d} \sum_{i \in [d]} \EE[X_i \pi_i(k)] 
        \\
        &= \frac{1}{d} \sum_{i \in [d]} \EE \left[\frac{\partial}{\partial X_i} \pi_i(k)\right]=-\frac{ak}{d} \sum_{i \in [d]} \EE[\pi_i(k)(1-\pi_i(k))] = -\frac{ak}{d} \EE[1-\|\pi(k)\|_2^2]
        \end{align*}
        From the calculation above,
    \begin{equation}\label{eq:F-master}
    F_d(a,k) \;=\; -\,\frac{a k}{d}\;\EE\!\big[\,1-\|\pi(k)\|_2^2\,\big],\qquad
    \pi(k)=\softmax(-kX).
    \end{equation}
    Write $X=\sqrt{a}\,Z$ with $Z\sim\cN(\pmb{0}_d, I_{d\times d})$ and set $\kappa := k\sqrt{a}$. Then
    \[
    \pi_i(k)\;=\;\frac{e^{-\kappa Z_i}}{\sum_{j=1}^d e^{-\kappa Z_j}}, 
    \qquad 
    \|\pi(k)\|_2^2 \;=\; \frac{\sum_{i=1}^d e^{-2\kappa Z_i}}{\Big(\sum_{j=1}^d e^{-\kappa Z_j}\Big)^2}.
    \]
    Let $Y_i:=e^{-\kappa Z_i}$ so that $Y_i>0$ are i.i.d.\ sub-exponential with 
    $\mu_1:=\EE Y_i=e^{\kappa^2/2}$ and $\mu_2:=\EE Y_i^2=e^{2\kappa^2}$.
    Define $S_1=\sum_{j=1}^d Y_j$ and $S_2=\sum_{i=1}^d Y_i^2$. 
    By Bernstein/Chernoff bounds for sub-exponential sums, there exist absolute constants $c_0,C_0>0$ (depending only on an a priori bound $\kappa\le\kappa_\star$) such that with probability at least $1-2e^{-c_0 d}$,
    \[
    \tfrac{1}{2}d\mu_1 \;\le\; S_1 \;\le\; 2d\mu_1
    \quad\text{and}\quad
    \tfrac{1}{2}d\mu_2 \;\le\; S_2 \;\le\; 2d\mu_2.
    \]
    On this event,
    \[
    \frac{1}{8}\cdot\frac{\mu_2}{\mu_1^2}\cdot\frac{1}{d} 
    \;\le\;
    \frac{S_2}{S_1^2} 
    \;=\; \|\pi(k)\|_2^2
    \;\le\;
    8\cdot\frac{\mu_2}{\mu_1^2}\cdot\frac{1}{d}
    \;=\;
    8\,\frac{e^{\kappa^2}}{d}.
    \]
    Taking expectations and absorbing the exponentially small tail yields constants $0<c_1\le C_1<\infty$ (depending only on $\kappa_\star$) for which
    \begin{equation}\label{eq:l2-mass}
    \frac{c_1}{d} \;\le\; \EE\|\pi(k)\|_2^2 \;\le\; \frac{C_1}{d}
    \qquad\text{for all } d\text{ large enough.}
    \end{equation}
    Combining \eqref{eq:l2-mass} with \eqref{eq:F-master}, we obtain
    \[
    -\;\frac{a k}{d}\Big(1-\frac{c_1}{d}\Big)
    \;\le\;
    F_d(a,k)
    \;\le\;
    -\;\frac{a k}{d}\Big(1-\frac{C_1}{d}\Big).
    \]
    In particular, for all $d\ge 2C_1$,
    \[
    -\;\frac{a k}{d} \;\le\; F_d(a,k) \;\le\; -\;\frac{1}{2}\,\frac{a k}{d}.
    \]
    Renaming absolute constants completes the proof.
    \end{proof}

\section{Proof of \Cref{fact:BM-swap}}
\label{appendix:pfBM}

\begin{proof}
By definition of external regret for $\cR_i$, we have~\eqref{eq:BM-Reg-i} for
each $i\in[d]$.

Fix any $\hat P = (\hat p_1 \mid \dots \mid \hat p_d) \in \Phi$.
The cumulative swap regret with respect to this transformation is
\begin{align}
  \sum_{t=1}^T \bigl(\langle \ell_t, \pi_t\rangle - \langle \ell_t, \hat P \pi_t\rangle\bigr)
  &\underset{(i)}{=} \sum_{t=1}^T
     \bigl(\langle \ell_t, P_t\pi_t\rangle - \langle \ell_t, \hat P \pi_t\rangle\bigr)= \sum_{t=1}^T \left[
        \sum_{i=1}^d \pi_{t,i} \langle \ell_t, p_{i,t}\rangle
        - \sum_{i=1}^d \pi_{t,i} \langle \ell_t, \hat p_i\rangle
     \right] \nonumber
     \\&= \sum_{t=1}^T \sum_{i=1}^d
        \bigl( \langle x_{i,t}, p_{i,t}\rangle - \langle x_{i,t}, \hat p_i\rangle \bigr) = \sum_{i=1}^d \sum_{t=1}^T
        \bigl( \langle x_{i,t}, p_{i,t}\rangle - \langle x_{i,t}, \hat p_i\rangle \bigr) \label{eq:Swap-regret-reordering} \\
  &\underset{(ii)}{\le} \sum_{i=1}^d \Reg_i((x_{i,t})_{t \in [T]}).\nonumber
\end{align}
Here (i) follows from the definition of $\pi_t$ and the fact that
$\pi_t = P_t\pi_t$, and (ii) follows from the definition of $\Reg_i$.
Since the identity holds for every $\hat P\in\Phi$, taking the maximum over
$\hat P$ gives \eqref{eq:BM-swap-decomposition}, and then
\[
  \SwapReg((\ell_t)_{t \in [T]})
  \le \sum_{i=1}^d \Reg_i((x_{i,t})_{t \in [T]}),
\]
which completes the proof.
\end{proof}

\section{Proof of \Cref{thm:new-architecture-stationary}}
\label{appendix:pfthm3}
\stationarynew*
We define partial sums $S_{i,t}:=\sum_{s=1}^t x_{i,s}\in\R^d$ with $S_{i,0}:=\pmb{0}_d$ for each head $i\in[d]$ and time $t\in[T]$. For each head $j\in[d]$ and time $t\in[T]$ we write $z_{j,t}$ for the pre-$\softmax$ logits and
\[
p_{j,t}=\operatorname{\softmax}(z_{j,t})\in\Delta([d])
\]
for the corresponding mixed action (the $j$-th column of the transition
matrix at time $t$).
Throughout we reuse the $\softmax$ Jacobian notation from \Cref{appendix:pfthm1}:
for $\pi=\operatorname{\softmax}(z)$, $J_{\sm}(z):=\partial\,\operatorname{\softmax}(z)/\partial z=\Diag(\pi)-\pi\pi^\intercal\in\R^{d\times d}$, so $J_{\sm}(z)^\intercal\pmb{1}_d=\pmb{0}_d$ and $\pmb{1}_d^\intercal J_{\sm}(z)=\pmb{0}_d^\intercal$.

\begin{proof}[{\color{blue}\textbf{Proof of the existence of the stationary point.}}]
We now impose the special parameter configuration from the statement:
for each $j\in[d]$ we write
\[
V^{(j)}=(V^{(j)}_1,\dots,V^{(j)}_d),\qquad V^{(j)}_r\in\R^{d\times d},
\]
and assume
\[
V^{(j)}_r=
\begin{cases}
-k I_{d\times d},& r=j,\\[2pt]
\pmb{O}_{d\times d},& r\neq j,
\end{cases}
\qquad
a^{(j)}=\pmb{0}_{d^2},\qquad
v_c^{(j)}=v\pmb{1}_d,
\]
for some $k>0$.
We also split $X_s$ into $d$ blocks of length $d$:
\[
X_s=\bigl(x_{1,s}^\intercal,\dots,x_{d,s}^\intercal\bigr)^\intercal,\qquad
x_{i,s}\in\R^d.
\]

At this configuration the quadratic term $X_s X_s^\intercal a^{(j)}$ vanishes
and $a^{(j)\intercal}X_s=0$, so
\[
z_{j,t+1}
=\sum_{s=1}^t \big(V^{(j)}X_s+v\pmb{1}_d\big)
=\sum_{s=1}^t\Big(\sum_{r=1}^d V^{(j)}_r x_{r,s} + v\pmb{1}_d\Big).
\]
Using the special form of $V^{(j)}$,
\[
z_{j,t+1}
= \sum_{s=1}^t \big(-k x_{j,s} + v\pmb{1}_d\big)
= -k S_{j,t} + t\,v\pmb{1}_d.
\]

Re-indexing time so that policies at round $t$ use the sum up to $t-1$
(as in \Cref{thm:stationary}), we obtain the simpler representation
\begin{equation}\label{eq:pj-single-head-form}
p_{j,t}
  = \operatorname{\softmax}\bigl(-k S_{j,t-1} + (t-1)v\pmb{1}_d\bigr),
  \qquad j\in[d],\ t\in[T].
\end{equation}
Thus, for each fixed $j$, the mapping
$(x_{j,1},\dots,x_{j,T})\mapsto (p_{j,t})_{t\le T}$ is exactly the
single-layer linear attention model from \Cref{thm:stationary}, with
$\ell_t$ there replaced by $x_{j,t}$ here. For a fixed realization of $(x_{i,t})$ the per-head regret $R_i$ satisfies
\[
R_i
= \sum_{t=1}^T \langle x_{i,t},p_{i,t}\rangle
  - \min_{q\in\Delta([d])}
     \Big\langle \sum_{t=1}^T x_{i,t}, q\Big\rangle
= \sum_{t=1}^T \langle x_{i,t},p_{i,t}\rangle - \min_{j\in[d]} S_{i,T,j}.
\]
By \eqref{eq:pj-single-head-form}, each $R_i$ has the same law as the
single-head regret $R_T(k)$ from \Cref{thm:stationary} (with
$\ell_t$ replaced by $x_{i,t}$). The inner minimum in $\mathcal L$ decouples across columns:
\begin{align*}
\min_{\hat P\in(\Delta([d]))^d}
   \sum_{t=1}^T\sum_{i=1}^d \langle x_{i,t},\hat P_i-p_{i,t}\rangle
&= \sum_{i=1}^d \min_{\hat p_i\in\Delta([d])}
     \sum_{t=1}^T \langle x_{i,t},\hat p_i - p_{i,t}\rangle\\
&= \sum_{i=1}^d\Big(
     \min_{\hat p_i}\sum_{t=1}^T \langle x_{i,t},\hat p_i\rangle
     - \sum_{t=1}^T \langle x_{i,t},p_{i,t}\rangle
   \Big)\\
&= \sum_{i=1}^d\Big(
     \min_{q\in\Delta([d])}\Big\langle S_{i,T},q\Big\rangle
     - \sum_{t=1}^T \langle x_{i,t},p_{i,t}\rangle
   \Big)\\
&= -\sum_{i=1}^d R_i.
\end{align*}
Hence we may write
\[
R_{\text{tot}}:=\SwapReg((x_{i,t})_{i\in[d],t\in[T]})
=\sum_{i=1}^d R_i,\qquad
\mathcal L((V^{(j)},a^{(j)},v_c^{(j)})_{j\in[d]})
=\mathbb E[R_{\text{tot}}^2].
\]
The $d$ random variables $R_1,\dots,R_d$ are independent and have the
same distribution as $R_T(k)$.

We now compute the gradient of $R_{\text{tot}}^2$ with respect to the logits
$z_{j,t}$. For each fixed $j$ the per-head regret $R_i$ with $i\neq j$ does not
depend on $p_{j,\cdot}$, hence
\[
\frac{\partial R_{\text{tot}}}{\partial p_{j,t}}
= \frac{\partial R_j}{\partial p_{j,t}}
= x_{j,t}.
\] Using $p_{j,t}=\softmax(z_{j,t})$ and the chain rule,
\[
\frac{\partial R_{\text{tot}}}{\partial z_{j,t}}
= J_{\sm}(z_{j,t})^\intercal x_{j,t},\qquad
\frac{\partial R_{\text{tot}}^2}{\partial z_{j,t}}
=2R_{\text{tot}}\,J_{\sm}(z_{j,t})^\intercal x_{j,t}.
\]
We abbreviate
\[
\Delta_{j,t}
:=\frac{\partial R_{\text{tot}}^2}{\partial z_{j,t}}
=2R_{\text{tot}}\,J_{\sm}(z_{j,t})^\intercal x_{j,t}\in\R^d.
\]

By dominated convergence ($\softmax$ is smooth and all Gaussian moments are finite), we may exchange gradient and
expectation. Thus for any perturbations $(\delta z_{j,t})_{j\in[d],t\in[T]}$,
\begin{equation}\label{eq:deltaL-swap}
\delta\mathcal L
= \mathbb E\Big[\sum_{j=1}^d\sum_{t=1}^T
   \langle \Delta_{j,t},\delta z_{j,t}\rangle\Big].
\end{equation}

We now express $\delta z_{j,t}$ in terms of perturbations of the
parameters of head $j$. For notational brevity we suppress the
superscript $(j)$ in this step and write $(V,a,v_c)$ for
$(V^{(j)},a^{(j)},v_c^{(j)})$.

From the definition of $z_{j,t+1}$ we obtain, for arbitrary perturbations
$(\delta V,\delta a,\delta v_c)$,
\begin{align*}
\delta z_{t+1}
&= \sum_{s=1}^t
   \Big(\delta V X_s X_s^\intercal a + V X_s X_s^\intercal \delta a\Big)
 + \sum_{s=1}^t
   \Big(\delta V X_s + v_c\,\delta a^\intercal X_s + \delta v_c\,a^\intercal X_s\Big)
 + \sum_{s=1}^t \delta v_c.
\end{align*}
At the special configuration $a=\pmb{0}_{d^2}$ this reduces to
\[
\delta z_{t+1}
= \sum_{s=1}^t \delta V X_s
 + \sum_{s=1}^t v_c\,\delta a^\intercal X_s
 + t\,\delta v_c.
\]
Restoring the head index $j$ and re-indexing time so that $t$ runs from
$1$ to $T$ (i.e., replacing $t+1$ by $t$), we have
\begin{equation}\label{eq:dz-decompose}
\delta z_{j,t}
= \underbrace{\sum_{s=1}^{t-1} \delta V^{(j)} X_s}_{(\ast_V)}
 +\underbrace{\sum_{s=1}^{t-1} v_c^{(j)}\,\delta a^{(j)\intercal} X_s}_{(\ast_a)}
 +\underbrace{(t-1)\delta v_c^{(j)}}_{(\ast_v)}.
\end{equation}

We now treat separately the three types of perturbations
$\delta v_c^{(j)}$, $\delta a^{(j)}$, and $\delta V^{(j)}$.

\medskip
\paragraph{(i) The $v_c^{(j)}$-direction (shift invariance).}

Using $(\ast_v)$ in \eqref{eq:dz-decompose},
\[
\delta z_{j,t}^{(v)} = (t-1)\delta v_c^{(j)},
\]
so \eqref{eq:deltaL-swap} gives
\[
\delta\mathcal L^{(v)}
= \sum_{j=1}^d\sum_{t=1}^T (t-1)\,\mathbb E\big[
   \langle \Delta_{j,t},\delta v_c^{(j)}\rangle\big].
\]
Write $\delta v_c^{(j)}=\delta v^{(j)}\pmb{1}_d$. Since
$\pmb{1}_d^\intercal J_{\sm}(z_{j,t})^\intercal=\pmb{0}_d^\intercal$ for all
$z_{j,t}$, we have
\[
\pmb{1}_d^\intercal \Delta_{j,t}
=2F\,\pmb{1}_d^\intercal J_{\sm}(z_{j,t})^\intercal x_{j,t}=0
\]
pointwise, and therefore each term in the sum vanishes. Hence
$\delta\mathcal L^{(v)}=0$ for all perturbations $(\delta v^{(j)})_{j\in[d]}$, and
\[
\nabla_{v_c^{(j)}}\mathcal L\Big|_{\text{conf.}}=\pmb{0}_d,\qquad j\in[d].
\]

\medskip
\paragraph{(ii) The $a^{(j)}$-direction at $a^{(j)}=0$.}

From $(\ast_a)$ in \eqref{eq:dz-decompose},
\[
\delta z_{j,t}^{(a)}
= \sum_{s=1}^{t-1} v_c^{(j)}\,\delta a^{(j)\intercal} X_s
= v\,\pmb{1}_d \sum_{s=1}^{t-1} \langle X_s,\delta a^{(j)}\rangle.
\]
Plugging into \eqref{eq:deltaL-swap},
\begin{align*}
\delta\mathcal L^{(a)}
&= \sum_{j=1}^d \mathbb E\Big[
    \sum_{t=1}^T \big\langle \Delta_{j,t},
        v\,\pmb{1}_d \sum_{s=1}^{t-1}\langle X_s,\delta a^{(j)}\rangle
      \big\rangle \Big]\\
&= v\sum_{j=1}^d \mathbb E\Big[
    \sum_{t=1}^T \big(\pmb{1}_d^\intercal \Delta_{j,t}\big)
        \sum_{s=1}^{t-1}\langle X_s,\delta a^{(j)}\rangle\Big].
\end{align*}
But we already observed that $\pmb{1}_d^\intercal \Delta_{j,t}=0$ for all
$(j,t)$, hence $\delta\mathcal L^{(a)}=0$ for every
$(\delta a^{(j)})_{j\in[d]}$ and therefore
\[
\nabla_{a^{(j)}}\mathcal L\Big|_{\text{conf.}}=\pmb{0}_{d^2},
\qquad j\in[d].
\]

\medskip
\paragraph{(iii) The $V^{(j)}$-direction at $V^{(j)}_j=-kI_{d\times d}$.}

It remains to analyse perturbations of the matrices $V^{(j)}$.
Recall that each $V^{(j)}$ is partitioned into $d$ blocks
$V^{(j)}_r\in\R^{d\times d}$, and similarly we write
$\delta V^{(j)}=(\delta V^{(j)}_1,\dots,\delta V^{(j)}_d)$.
From $(\ast_V)$ in \eqref{eq:dz-decompose} we get
\[
\delta z_{j,t}^{(V)}
= \sum_{s=1}^{t-1}\delta V^{(j)}X_s
= \sum_{r=1}^d \delta V^{(j)}_r S_{r,t-1},
\]
since $X_s=(x_{1,s}^\intercal,\dots,x_{d,s}^\intercal)^\intercal$.

Substituting into \eqref{eq:deltaL-swap} and using the Frobenius inner
product $\langle A,B\rangle_F=\operatorname{tr}(A^\intercal B)$,
\begin{align*}
\delta\mathcal L^{(V)}
&= \sum_{j=1}^d \mathbb E\Big[
    \sum_{t=1}^T \Big\langle \Delta_{j,t},
      \sum_{r=1}^d \delta V^{(j)}_r S_{r,t-1}\Big\rangle\Big]\\
&= \sum_{j=1}^d \sum_{r=1}^d
    \Big\langle
      \underbrace{\mathbb E\Big[\sum_{t=1}^T
        \Delta_{j,t} S_{r,t-1}^\intercal\Big]}_{=:G_{j,r}(k)}\ ,\
      \delta V^{(j)}_r\Big\rangle_F.
\end{align*}
Thus
\[
\nabla_{V^{(j)}_r}\mathcal L\Big|_{\text{conf.}}=G_{j,r}(k)\in\R^{d\times d}.
\]

We now exploit the symmetries of the Gaussian model.

\smallskip
\emph{Case $r\neq j$.}
At the special configuration $V^{(j)}_r=\pmb{O}_{d\times d}$ for $r\neq j$,
the logits $z_{j,t}$ and hence $\Delta_{j,t}$ depend only on the block
process $(x_{j,s})_{s\le T}$, while the random walk $(S_{r,t})_{t\le T}$
is built from $(x_{r,s})_{s\le T}$ and is independent of
$(x_{j,s})_{s\le T}$. Therefore
\[
G_{j,r}(k)
= \mathbb E\Big[\sum_{t=1}^T \Delta_{j,t}\Big]\,
  \mathbb E\big[S_{r,t-1}^\intercal\big]
= \pmb{0}_{d\times d},
\]
because $\mathbb E[S_{r,t-1}]=\pmb{0}_d$ for all $r,t$.

\smallskip
\emph{Case $r=j$.}
Fix $j\in[d]$ and consider the action of coordinate permutations on
the action coordinates of head $j$: for a permutation matrix $P$,
\[
x_{j,t}\mapsto Px_{j,t},\quad
S_{j,t-1}\mapsto PS_{j,t-1}.
\]
By \eqref{eq:pj-single-head-form}, $V^{(j)}_j=-kI_{d\times d}$, $a^{(j)}=0$ and
$v_c^{(j)}=v\pmb{1}_d$, and by permutation-equivariance of $\softmax$,
we also have
\[
z_{j,t}\mapsto Pz_{j,t},\quad
p_{j,t}\mapsto Pp_{j,t},\quad
\Delta_{j,t}\mapsto P\Delta_{j,t}.
\]
Since the joint law of $(x_{j,t})_{t\le T}$ is invariant under coordinate
permutations and independent of the other heads, the pair
$(\Delta_{j,t},S_{j,t-1})_{t\le T}$ has the same distribution as
$(P\Delta_{j,t},PS_{j,t-1})_{t\le T}$. Consequently,
\[
G_{j,j}(k)
= \mathbb E\Big[\sum_{t=1}^T \Delta_{j,t} S_{j,t-1}^\intercal\Big]
= \mathbb E\Big[\sum_{t=1}^T (P\Delta_{j,t})(PS_{j,t-1})^\intercal\Big]
= P G_{j,j}(k) P^\intercal
\]
for all permutation matrices $P$. By \Cref{lem:permutation-commutant},
there are scalars $\alpha(k),\beta(k)\in\R$ such that
\[
G_{j,j}(k)=\alpha(k)I_{d\times d}
+\beta(k)\pmb{1}_d\pmb{1}_d^\intercal.
\]
Moreover, $\pmb{1}_d^\intercal\Delta_{j,t}=0$ for every $j,t$ because
$J_{\sm}(z_{j,t})^\intercal\pmb{1}_d=\pmb{0}_d$, so
$\pmb{1}_d^\intercal G_{j,j}(k)=\pmb{0}_d^\intercal$. Hence
\[
G_{j,j}(k)=\alpha(k)\Pi_0,\qquad
\Pi_0:=I_{d\times d}
-\frac{1}{d}\pmb{1}_d\pmb{1}_d^\intercal.
\]
The scalar is the same for every $j$ by exchangeability of the heads and
blocks at the stated configuration.

Combining the two cases, we have shown that at the stated configuration
\[
\nabla_{V^{(j)}_r}\mathcal L
=
\begin{cases}
\alpha(k)\Pi_0,& r=j,\\[2pt]
\pmb{O}_{d\times d},& r\neq j,
\end{cases}
\qquad j\in[d].
\]
Together with the $v_c^{(j)}$- and $a^{(j)}$-direction results above, this implies that, for every fixed $k$, the only
possibly non-zero component of the gradient of $\mathcal L$ at the
special configuration lies in the one-dimensional direction in which
each $V^{(j)}_j$ is scaled by the zero-sum projector $\Pi_0$.
\end{proof}

Now, we show that the stationary point is a scalar $k^\star$ such that $\cL'(k^\star)=0$. 

\begin{proof}[{\color{blue}\textbf{Proof of the scale of the stationary point.}}]

Define the one-dimensional function
\[
\Phi(k):=\mathcal L\bigl((V^{(j)},a^{(j)},v_c^{(j)})_{j\in[d]}\bigr)
\]
when, for each $j\in[d]$,
\[
V^{(j)}_r=
\begin{cases}
-kI_{d\times d},& r=j,\\
\pmb{O}_{d\times d},& r\neq j,
\end{cases}\qquad
a^{(j)}=\pmb{0}_{d^2},\quad v_c^{(j)}=v\pmb{1}_d.
\]

By the chain rule,
\[
\Phi'(k)
= \sum_{j=1}^d \Big\langle
    \nabla_{V^{(j)}}\mathcal L,\,
    \frac{\partial V^{(j)}}{\partial k}
  \Big\rangle_F
= \sum_{j=1}^d \Big\langle
    \alpha(k)\Pi_0,\,
   -I_{d\times d}\Big\rangle_F
= -d(d-1)\alpha(k),
\]
since $\partial V^{(j)}_j/\partial k=-I_{d\times d}$ and
$\partial V^{(j)}_r/\partial k=\pmb{O}_{d\times d}$ for $r\neq j$.
In particular,
\[
\alpha(k)=0\quad\Longleftrightarrow\quad \Phi'(k)=0.
\]

Let $R(k)$ denote a random variable with the law of $R_T(k)$ for the
single-head model, and note that $R_1(k,X),\dots,R_d(k,X)$ are
independent copies of $R(k)$ for every fixed $k$. Therefore
\begin{align}
\Phi(k)
= \EE\bigl[R_{\mathrm{tot}}(k,X)^2\bigr]
&= \EE\Big[\Big(\sum_{i=1}^d R_i(k,X)\Big)^2\Big]\notag\\
&= \sum_{i=1}^d \EE[R_i(k,X)^2]
   + 2\sum_{1\le i<j\le d} \EE[R_i(k,X)R_j(k,X)]\notag\\
&= d\,\EE[R(k)^2]
   + d(d-1)\bigl(\EE[R(k)]\bigr)^2.
\label{eq:Phi-decomposition}
\end{align}
\begin{proposition}
  \label{prop:R-independent-of-k}
The expectation of the single-head regret is independent of $k$, i.e., $\EE[R(k)]$ is independent of $k$.
\end{proposition}
\begin{proof}
 Write the single-head process as
$\{x_t\}_{t=1}^\intercal$ with $x_t\sim\cN(\pmb{0}_d, I_{d\times d})$ i.i.d., and denote
$S_t:=\sum_{s=1}^t x_s$ and
\[
\pi_t(k):=\softmax(-k S_{t-1} + (t-1)v\pmb{1}_d),
\]
so that
\[
R(k) = \sum_{t=1}^T \langle x_t,\pi_t(k)\rangle
       - \min_{q\in\Delta([d])}\Big\langle \sum_{t=1}^T x_t,\ q\Big\rangle.
\]

First, \[
C(X):=\min_{q\in\Delta([d])}\Big\langle \sum_{t=1}^T x_t,\ q\Big\rangle
\]
depends only on the losses $(x_t)_{t\le T}$ and is completely
independent of the parameter $k$. Moreover, note that $\pi_t(k)$ is measurable with respect
to the $\sigma$-algebra generated by $(x_1,\dots,x_{t-1})$ and does not
depend on $x_t$. Using tower property and the fact that $x_t$ is
independent of $(x_s)_{s<t}$ and has mean zero, we obtain
\begin{align*}
\EE\bigl[\langle x_t,\pi_t(k)\rangle\bigr]
&= \EE\Big[\EE\bigl[\langle x_t,\pi_t(k)\rangle
     \,\big|\, x_1,\dots,x_{t-1}\bigr]\Big]\\
&= \EE\Big[\big\langle \EE[x_t\mid x_1,\dots,x_{t-1}],\ \pi_t(k)\big\rangle\Big]
 = \EE\big[\langle 0,\pi_t(k)\rangle\big] = 0.
\end{align*}
Summing over $t$ gives
\[
\EE\Big[\sum_{t=1}^T \langle x_t,\pi_t(k)\rangle\Big] = 0
\quad\text{for all }k\in\R.
\]

Therefore
\[
\EE[R(k)]
= 0 - \EE[C(X)]
= -\,\EE\Big[\min_{q\in\Delta([d])}
              \Big\langle \sum_{t=1}^T x_t,\ q\Big\rangle\Big],
\]
which is a finite constant independent of $k$.
\end{proof}

From \eqref{eq:Phi-decomposition} and \Cref{prop:R-independent-of-k} we
obtain
\begin{align*}
\Phi'(k)
&= d\,\frac{d}{dk}\EE[R(k)^2]
 + d(d-1)\,\frac{d}{dk}\bigl(\EE[R(k)]^2\bigr)\\
&= d\,\frac{d}{dk}\EE[R(k)^2]
 + d(d-1)\cdot 2\,\EE[R(k)]\,\frac{d}{dk}\EE[R(k)]\\
&= d\,\frac{d}{dk}\EE[R(k)^2]
\quad\text{(since $\frac{d}{dk}\EE[R(k)]=0$)}.
\end{align*}
Recalling the external-regret objective
\[
L_{\mathrm{ext}}(k) := \EE[R(k)^2],
\]
this identity can be succinctly written as
\begin{equation}\label{eq:Phi-prime-vs-Lext}
\Phi'(k) = d\,L_{\mathrm{ext}}'(k)\quad\text{for all }k\in\R.
\end{equation}

\medskip
\textbf{Sign of $\Phi'(k)$ and existence of $k^\star$.}

The single-head analysis in \Cref{thm:smallk,thm:positivity} shows that
there exist absolute constants $c_1,c_2>0$ (independent of $d,T$) such
that: For all sufficiently large $T,d$ and all
  \(
    0 < k \le c_1 / \sqrt{Td},
  \)
  one has
  \[
  L_{\mathrm{ext}}'(k) < 0.
  \]
  For all $T,d\ge 2$ and all
  \(
    k \ge c_2 / \sqrt{T}\,\mathrm{polylog}(d),
  \)
  one has
  \[
  L_{\mathrm{ext}}'(k) > 0.
  \]
Combining these with \eqref{eq:Phi-prime-vs-Lext}, we obtain
\[
\Phi'(k) < 0
\quad\text{for }k\le c_1/\sqrt{Td},
\qquad
\Phi'(k) > 0
\quad\text{for }k\ge c_2/\sqrt{T}\,\mathrm{polylog}(d).
\]

The function $k\mapsto \Phi(k)$ is smooth ($\softmax$ and the architecture
are smooth and all Gaussian moments are finite), hence $\Phi'$ is
continuous. By the intermediate value theorem there exists
\[
k^\star\in\Bigl(\tilde\Theta\Bigl(\tfrac{1}{\sqrt{Td}}\Bigr),
               \tilde\Theta\Bigl(\tfrac{1}{\sqrt{T}}\Bigr)\Bigr)
\]
such that $\Phi'(k^\star)=0$.
\end{proof}

\section{Additional Background}
\label{appendix:background}

\subsection{Literature Review}
\paragraph{Transformers \& In-context-learning.}
LLMs are now predominantly built upon the self-attention architecture \citep{vaswani2017attention}. A central reason for their success is their remarkable ability to perform \emph{in-context learning} (ICL): Transformers can construct new predictors directly from sequences of labeled examples provided in the input, without any parameter updates. This mechanism enables powerful \emph{few-shot learning} capabilities \citep{brown2020language,garg2022can,min2022rethinking}, which in turn have motivated a rapidly growing theoretical literature on ICL. For certain training losses, the \emph{minimizer} of a single-layer Transformer is equivalent to performing one step of gradient descent for linear regression \citep{ahn2023transformers,zhang2023trained,mahankali2023one}. For attention models with non-linear operators, prior work has provided provable guarantees for ICL---either through training-dynamics analyses or partial characterizations of stationary points. For example, \citet{ahn2023transformers} studied transformers with a ReLU operator, and \citet{cheng2024transformersimplementfunctionalgradient} analyzed the $\softmax$ operator from a functional-gradient perspective. In addition, \citet{chen2024trainingdynamicsmultiheadsoftmax} examined the training dynamics of multi-head Softmax attention in the context of ICL. However, none of these works characterizes the exact configuration of the stationary point of a transformer with a non-linear operator. \citet{cheng2024transformersimplementfunctionalgradient} derived a stationary-point result but did not identify the precise configuration corresponding to a specific (functional) gradient-descent stepsize; \citet{chen2024trainingdynamicsmultiheadsoftmax} provided a dynamical analysis for multi-head Softmax attention, but likewise did not interpret or characterize the stationary-point structure of transformers with non-linear operators. In contrast, our analysis \emph{explicitly characterizes} the stationary-point configuration of a Transformer equipped with a non-linear operator, thereby yielding a provable guarantee for the no-regret behavior induced by such architectures---an independently interesting theoretical contribution. More recently, ICL has also been generalized to decision-making settings. Works such as \citet{laskin2022context,lee2023supervised,lin2023transformers} demonstrated that supervised pre-training can endow Transformers with in-context reinforcement learning abilities, enabling them to solve stochastic bandits and Markov decision processes. In comparison, regret-loss training focuses on \emph{online learning} settings that may be arbitrary or even \emph{adversarial}, as well as \emph{game-theoretic} environments where agents interact strategically.

\subsection{Known Algorithms for No-Regret Learning}
\paragraph{Follow-the-Regularized-Leader (FTRL).}
FTRL generates the sequence
\[
    x_{t+1}
    \;=\;
    \argmin_{x\in\Pi} \left\{ \eta \sum_{s=1}^{t} \ell_s(x) + R(x) \right\},
\]
where $R$ is a strongly convex regularizer and $\eta>0$ is the learning rate. With the negative-entropy regularizer $R(x)=\sum_i x_i\log x_i$, FTRL reduces to the multiplicative-weights update (or exponential weights).

\paragraph{Smoothed Fictitious Play as Entropy-Regularized FTRL.}
Smoothed fictitious play is the update rule
\[
    x_{t+1}
    \;\propto\;
    \exp\!\left(
        -\eta \sum_{s=1}^{t} \ell_s
    \right),
\]
which is exactly FTRL with entropy regularization.
The smoothing (logit or softmax) arises from the convex conjugate of the negative entropy. If stepsize $\eta = \Theta(1/\sqrt{T})$, smoothed fictitious play has a no-regret guarantee \citep{cesa2006prediction}. In this paper, we interchangeably use the terms \emph{smoothed fictitious play} and \emph{FTRL with entropy regularization}.

\section{Use of AI Assistance}
\label{appendix:ai-assistance}

The authors used OpenAI Codex to assist with code implementation, debugging,
refactoring, and the linguistic and stylistic editing of the manuscript. All
research questions, theoretical ideas, theorem statements, proofs,
experimental designs, validation procedures, interpretations, and conclusions
were developed and verified by the authors. In particular, the scientific
results reported in this paper are the authors' own work; Codex was used as a
coding and editing aid and not as a source of scientific claims. The authors
reviewed all AI-assisted material and take full responsibility for the content
of the paper and the accompanying code.

\end{document}